%% file: main.tex
\documentclass[twoside]{article}

\usepackage[preprint]{aistats2027}

\usepackage[round]{natbib}

\usepackage{color}
\usepackage{subfigure}
\usepackage{graphicx}
\usepackage{hyperref}
\hypersetup{colorlinks=true, linkcolor=[rgb]{0,0,0.55}, citecolor=[rgb]{0,0,0.55}, urlcolor=[rgb]{0,0,0.55}}
\usepackage{amsthm}
\usepackage{algorithm}
\usepackage{algpseudocode}
\usepackage{mathtools}
\newcommand{\cJ}{\mathcal{J}}

\newcommand{\mask}{\operatorname{mask}}
\newcommand{\TC}{\mathrm{TC}}
\providecommand{\diag}{\operatorname{diag}}
\newcommand{\eqby}[1]{\mathrel{\mathop{\kern0pt =}\limits_{\mathclap{\text{\scriptsize #1}}}\mskip\thickmuskip}}
\algrenewcommand\algorithmicrequire{\textbf{Input:}}
\algrenewcommand\algorithmicensure{\textbf{Output:}}

\theoremstyle{plain}
\newtheorem{theorem}{Theorem}[section]
\newtheorem{lemma}[theorem]{Lemma}
\newtheorem{proposition}[theorem]{Proposition}

\theoremstyle{definition}

\newtheorem{assumption}[theorem]{Assumption}

\theoremstyle{remark}
\newtheorem{remark}[theorem]{Remark}

\input{math_commands}

\begin{document}

\twocolumn[

\aistatstitle{Self-sufficient Independent Component Analysis for Demixing Flows}

\aistatsauthor{ Song Liu }

\aistatsaddress{ University of Bristol } ]

\begin{abstract}
We study the problem of learning disentangled signals from data using non-linear Independent Component Analysis (ICA).  
Motivated by advances in self-supervised learning, we propose to learn self-sufficient signals:
Given the remaining values of a recovered signal, observing other signals should not change the conditional distribution of its missing value. We formulate this problem as the minimization of a conditional KL divergence.
Our algorithm is prior-free and likelihood-free in the sense that it prescribes neither parametric source densities nor an observation likelihood.
To tackle the KL divergence minimization problem, 
we propose a sequential algorithm that learns a de-mixing flow model at each iteration, and prove local descent of the total correlation for its idealized Wasserstein-gradient-flow variant with exact velocities and a population projection condition. This approach completely avoids the unstable adversarial training, a common issue in minimizing the KL divergence.
Experiments on toy and real-world datasets show the effectiveness of our method. 
\end{abstract}

\section{Introduction}
Independent Component Analysis (ICA) learns a de-mixing function that recovers independent source signals from their mixture \citep{hyvarinen_independent_2000,hyvarinen_nonlinear_2023}, and is a classic approach to learning disentangled features \citep{dinh_nice_2015,chen_infogan_2016,Wang2024}. While linear ICA is well understood \citep{amari_new_1995,bell_information-maximization_1995}, independence alone does not identify the sources under non-linear mixing, and existing methods exploit additional structure, such as autocorrelation \citep{hyvarinen_nonlinear_2017}, non-stationarity \citep{hyvarinen_unsupervised_2016}, or an observed auxiliary variable \citep{hyvarinen_nonlinear_2019,khemakhem_variational_nodate}.

Separating sources by minimizing the dependence between recovered signals is a classic idea in ICA \citep{suzuki_least-squares_2011,babaie-zadeh_general_2005}; for sequences, a natural dependence measure is the total correlation between whole signals. This criterion is, however, hard to minimize in practice. Likelihood-based methods require explicit prior and likelihood models, which either restrict the model's flexibility or make the likelihood intractable, while direct estimation of the mutual information typically relies on adversarial training of a neural estimator \citep{brakel_learning_2017,belghazi_mutual_2018}, which can be numerically unstable. Moreover, for sequences, the gradient of the total correlation involves densities over entire sequences and is intractable to estimate.

Motivated by self-supervised learning, which trains a model to recover masked values from the remaining ones \citep{he_masked_2022,shi2022adversarial,tashiro2021csdi}, we propose Self-sufficient ICA (SICA): given its own remaining values, the conditional distribution of a missing value should be unchanged by observing the other signals. This self-sufficiency condition yields a factorization of conditional densities at each time point, and it makes the problem above tractable: under a stationarity assumption, the gradient of the total correlation decomposes into per-time-point conditional density ratios, which incorporate the context of each time point and can be estimated from a single sequence. We also show that self-sufficiency identifies the sources under non-linear mixing when the signals depend on their own context in sufficiently different ways (Theorem~\ref{thm.ident}).

To minimize the resulting objective, we build on ODE-based generative models \citep{chen_neural_2018,lipman_flow_2023,liu_flow_2022,yi23mono}. We observe that some mixing processes can be interpreted as solutions of ODEs, and therefore learn ODEs as de-mixing functions (``de-mixing flows''). To avoid adversarial training, we learn the de-mixing flow iteratively, using either a Wasserstein gradient flow (WGF) or a rectified flow at each iteration, without imposing any prior or likelihood model. Although derived from self-sufficiency, the WGF variant turns out to follow the nonparametric gradient of the total correlation between the recovered sequences \citep{babaie-zadeh_differential_2004}, yielding local descent under the population projection condition of Theorem~\ref{thm.tc.decrease}. Experiments on autoregressive, image and speech data demonstrate the effectiveness of our method.

\section{From Independence to Self-sufficiency}
\label{sec.ass}

\textbf{Notations}: $x, \vx, \mX$ are scalar, vector and matrix. $\rx, \rvx, \rmX$ are random scalars, vectors and matrices. $p(\rx)$ is the probability of random variable $\rx$ and $p(\rx = x)$ is the probability density function of random variable $\rx$ evaluated at $x$. $\mX_{i, :}$ represents the $i$-th row of a matrix $\mX$ and $\mX_{-i, :}$ represents all rows of $\mX$ except the $i$-th row. $\mask_t(\mX)$ represents $\mX$ with its $t$-th column replaced by a vector of missing values $\mathrm{NaN}$; unlike deleting the column, the masked column keeps its position. We write $\tilde{\R} := \R \cup \{\mathrm{NaN}\}$. For a random matrix $\rmX$ with rows $\rmX_{1,:}, \dots, \rmX_{d,:}$, we define the total correlation between its rows,
\begin{align}
    \label{eq.tc.def}
  \TC(\rmX) := \KL\Big[p(\rmX) \,\Big\Vert\, \prod_{i=1}^{d} p(\rmX_{i,:})\Big],
\end{align}
i.e., the total correlation between the signals, each viewed as a whole sequence. $\TC(\rmX) = 0$ if and only if the rows of $\rmX$ are mutually independent. 

\subsection{Independent Component Analysis}
\label{sec.ica}
Independent Component Analysis (ICA) \citep{hyvarinen_independent_2000,hyvarinen_nonlinear_2023}
assumes that some original signal $\rvs \in \sR^d$ is transformed through an \textbf{invertible}, possibly non-linear mixing function $\vf: \sR^d \to \sR^d$, i.e., $\rvx = \vf(\rvs)$,
where $\rvs = [\rs_1, \dots, \rs_d]$ is a vector of \emph{independent components} and $\rvx$ is the observed signal.
The task of ICA is to learn a de-mixing function $\vg: \sR^d \to \sR^d$ such that $\rvz := \vg(\rvx)$ recovers $\rvs$.
The optimal $\vg$ recovers the original signal, so $\rvz = \rvs$. If we restrict $\vg$ to the family of bijective linear functions, we recover the linear ICA problem.

Classical ICA algorithms seek a $\vg$ that makes the components of $\rvz$ independent, by maximum likelihood with independent priors \citep{dinh_nice_2015} or by minimizing an estimate of their mutual information \citep{suzuki_least-squares_2011,brakel_learning_2017}.

\subsection{Self-sufficient ICA}
One recent development in non-linear ICA is improving identifiability via an auxiliary variable. 
Consider the observed signal $\rvx$ generated using the following formula:
\begin{equation}
\label{eq.example}
\begin{gathered}
 \ru \sim \mathrm{uniform}(0, 2\pi), \quad
 \rs_2 := h_1(\ru) + \epsilon, \\
 \rs_1 := h_2(\ru) + \epsilon', \quad
 \rvx = \begin{bmatrix} 1 & .5 \\ .5 & 1 \end{bmatrix}
 \begin{bmatrix} \rs_1 \\ \rs_2 \end{bmatrix},
\end{gathered}
\end{equation}
where $\epsilon, \epsilon'$ are some random independent noise. $\rvx$ is a linear mixture of $\rs_1$ and $\rs_2$. ICA cannot recover $\rs_1$ and $\rs_2$ perfectly using $\rvx$ alone, as the independence assumption in ICA does not hold: $\rs_1$ and $\rs_2$ are not independent due to their associations with the common variable $\ru$. See Figure~\ref{fig:illus.example} for an example of such a dataset.
\begin{figure*}[t]
    \centering
    \includegraphics[width=\textwidth]{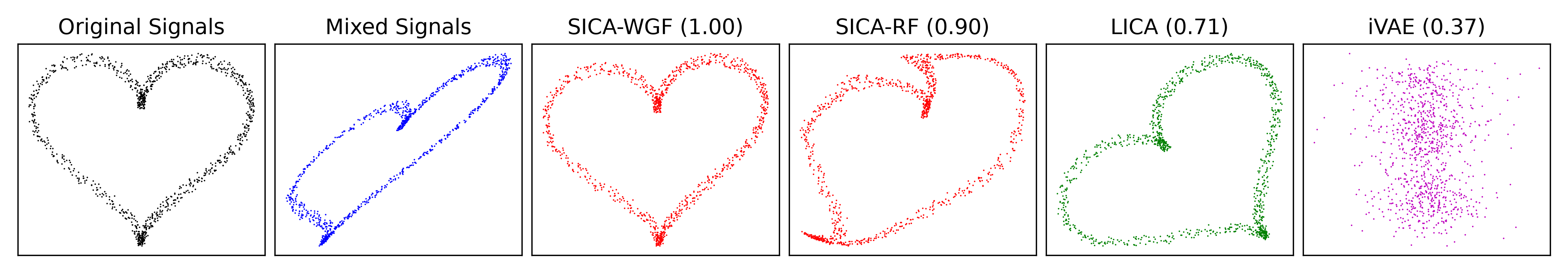}
    \caption{Two dependent signals mixed by a linear mixing function. Both a linear ICA (LICA, \citealp{suzuki_least-squares_2011}) and a non-linear ICA (iVAE, \citealp{khemakhem_variational_nodate}) fail to de-mix the signals (LICA returns a tilted heart, and iVAE loses its shape), whereas the proposed method recovers the heart (SICA-WGF; SICA-RF with some distortion in this run). The MCC with the true signals is given in brackets; see Appendix~\ref{app.heart} for details, results over 20 runs and further baselines.}
    \label{fig:illus.example}
\end{figure*}

Auxiliary-variable ICA \citep{hyvarinen_nonlinear_2019} relaxes independence to conditional independence given an observed auxiliary variable $\rvu$: in \eqref{eq.example}, $\rs_1$ and $\rs_2$ are dependent but conditionally independent given $\ru$. With past observations as $\rvu$, this recovers PCL \citep{hyvarinen_nonlinear_2017}, and \citet{halva_disentangling_2021} extend it to structured temporal dependence.

\begin{figure}[t]
    \centering
    \includegraphics[width=.9\columnwidth]{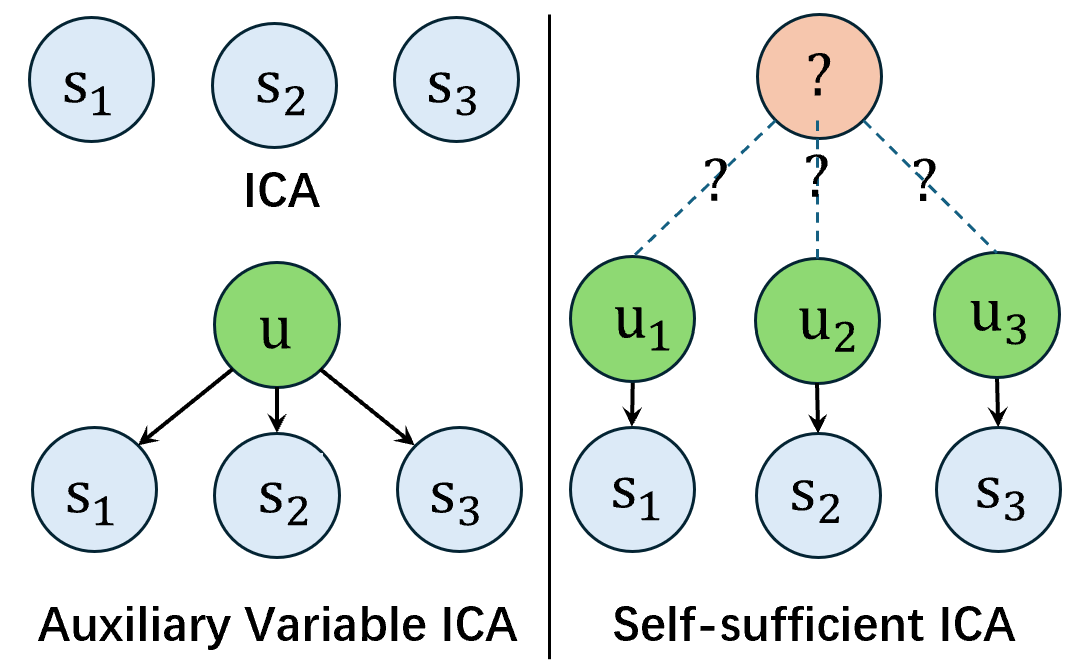}
    \caption{Graphical models of ICA, auxiliary-variable ICA and SICA. In SICA, each signal $\rs_i$ has its own auxiliary variable $\rvu_i$, and the $\rvu_i$ may depend on each other arbitrarily.}
    \label{fig:placeholder}
\end{figure}

We study a specific setting, \emph{self-sufficient ICA} (Figure~\ref{fig:placeholder}): each signal $\rs_i$ has its own auxiliary variable $\rvu_i$, rather than a shared one, and $\rvu_i$ is sufficient for predicting $\rs_i$, i.e., no other pair $(\rs_j, \rvu_j)$, $j \neq i$, improves the prediction.

Specifically, we assume that 
\begin{align}
    \label{eq.aica.independnece}
    \rs_i \ind (\rvs_{-i}, \rvu_{-i}) \mid \rvu_i.  
\end{align}
The independence assumption in \eqref{eq.aica.independnece} is equivalent to the following density factorization  
\begin{align}
    \label{eq.aica}
    p(\rvs \mid \rmU) = \prod_i p(\rs_i \mid \rvu_i),
\end{align}
where $\rmU = [\rvu_1, \dots, \rvu_d]^\top$. The proof involves applying the chain rule for densities; details can be found in Appendix~\ref{app:proof_aica}.
This is stronger than conditional independence given a shared $\rmU$, and weaker than independence of the pairs $(\rs_i, \rvu_i)$. We call this property \emph{self-sufficiency} and the resulting method Self-sufficient ICA (SICA).

\subsection{Sequence SICA}
\label{sec.ssica}
Consider a random, multi-dimensional signal $\rvs_t$, indexed by $t \in [T]$. 
We define 
$\rmS := [\rvs_t]_{t \in [T]} \in \sR^{d \times T},$ where $\rvs_t= [\rs_{1,t}, \dots, \rs_{d, t}]^\top$ is the $t$-th column of $\rmS$. 

Let us define the auxiliary variable of $\rvs_t$ as
\[\rmU_t := \mask_t(\rmS) = [\rvs_1, \dots, \rvs_{t-1}, \mathrm{NaN}, \rvs_{t+1}, \dots, \rvs_{T}],\] with rows $\rvu_{i,t} := [\rmU_t]_{i,:}$. 
$\rmU_t$ is the same as $\rmS$ except its $t$-th column is replaced by a vector of missing values (NaN) and is the \emph{context} of the signal $\rvs_t$. 

Then the SICA assumption in \eqref{eq.aica.independnece} in this context can be expressed as
\begin{align}
\label{eq:sica:independence}
 \rs_{i,t} \ind \rmS_{-i, :} \mid \rvu_{i,t}, \quad t \in [T].
\end{align}
It means at time $t$, given the remainder of $i$-th signal, the $i$-th signal is independent of any other signals at any time.
See Figure~\ref{fig.seq.sica} in Appendix~\ref{app.seq.illus} for a visualization of \eqref{eq:sica:independence}.

Equivalently, we can express the assumption as a factorization of density functions, similar to what we did in \eqref{eq.aica}:
\begin{align}
\label{eq.sica}
p(\rvs_t \mid \rmU_t) = \prod_{i=1}^{d} p(\rs_{i,t} \mid \rvu_{i,t}), \quad t \in [T].
\end{align}

\begin{remark}
\label{rem.sica.indep}
If $\rmS$ has a strictly positive density, \eqref{eq:sica:independence} holding \emph{for all} $i \in [d]$ and $t \in [T]$ is equivalent to $\TC(\rmS) = 0$, i.e., the source signals are independent as whole sequences. This follows from the intersection property of conditional independence \citep{lauritzen_graphical_1996}; see Proposition~\ref{prop.tc.relation} and its proof. 
\end{remark}

Following the standard ICA setting, 
we assume that $\rvx_t = \vf(\rvs_t), t \in [T]$.
The de-mixing function is $\vg = \vf^{-1}$ applied to each column, which recovers $\rvs_t = \vg(\rvx_t)$ for all $t$. We denote the recovered signals by $\rvz_t$, write $$\rmZ := [\rvz_1, \dots, \rvz_T],$$ and aim for $\rvz_t = \rvs_t, \forall t.$ Analogously to $\rmU_t$, $\rmC_t := \mask_t(\rmZ)$ is the \emph{context} of the recovered signal $\rvz_t$, with rows $\rvc_{i,t} := [\rmC_t]_{i,:}$. %

One can easily extend our assumption to multi-dimensional indices, which enables applications in image or video processing and spatial data analysis. In the rest of the paper, we focus on SICA applied to sequence data.

\paragraph{Identifiability.}
We now show that, for a de-mixing $\vg$, the SICA assumption identifies the sources under non-linear mixing, provided that the signals depend on their own context in sufficiently different ways. 

\begin{assumption}
\label{ass.cross}
For every $\vs \in \R^d$, 
there exist $N>0$, $\vs' \in \R^d$ and time points $t_1, \dots, t_N \neq t$ 
such that the cross-derivative vectors 
\[
\boldsymbol{\omega}_i := \biggl(\frac{\partial^2 \log p(\rs_{i,t} \mid \rvu_{i,t})}{\partial \rs_{i,t}\, \partial \rs_{i,t_n}}\biggr)_{n=1}^{N} \in \R^N, \qquad i \in [d],
\]
are nonzero and no two of them are proportional, 
where $\rs_{i,t}$ and $\rs_{i,t_n}$ are evaluated at $s_i$ and $s'_i$, respectively, and the remaining entries of $\rvu_{i,t}$ at some values. 
 
\end{assumption}

\begin{theorem}[Identifiability of SICA]
\label{thm.ident}
Let $\rmS$ have a positive continuous density whose conditional densities $p(\rs_{i,t} \mid \rvu_{i,t})$ are positive and twice continuously differentiable, and let $\rmS$ satisfy \eqref{eq.sica} at time $t$ and Assumption~\ref{ass.cross}. Let $\rvx_{t'} = \vf(\rvs_{t'})$ and $\rvz_{t'} = \vg(\rvx_{t'})$ for all $t' \in [T]$, with $C^1$ diffeomorphisms $\vf, \vg: \R^d \to \R^d$. If $p(\rvz_t \mid \rmC_t) = \prod_{i=1}^d p(\rz_{i,t} \mid \rvc_{i,t})$, then there are a permutation $\sigma$ of $[d]$ and diffeomorphisms $\kappa_i: \R \to \R$ such that $\rs_{i,t'} = \kappa_i(\rz_{\sigma(i),t'})$ for all $i \in [d]$ and $t' \in [T]$.
\end{theorem}

The proof (Appendix~\ref{app.ident}) takes the mixed derivative of $\log p(\rvz_t \mid \rmC_t)$ in $\rz_{j,t}$ and $\rz_{l,t'}$, $j \neq l$, which vanishes by the SICA factorization of $\rmZ$; the Jacobian term of the change of variables depends on $\rvz_t$ only and drops out. What remains states that $\mJ_\vk(\rvz_t)^\top \mG_{t'}\, \mJ_\vk(\rvz_{t'})$ is diagonal, where $\mJ_\vk$ is the Jacobian of $\vk := (\vg \circ \vf)^{-1}$ and $\mG_{t'}$ is the diagonal matrix of the cross-derivatives $\partial^2 \log p(\rs_{i,t} \mid \rvu_{i,t}) / \partial \rs_{i,t}\, \partial \rs_{i,t'}$, $i \in [d]$. Due to Assumption~\ref{ass.cross}, these equations force $\mJ_\vk(\rvz_t)$ to be a permutation matrix times a diagonal matrix, by the uniqueness of joint diagonalization, which also underlies second-order methods such as AMUSE \citep{tong_indeterminacy_1991}.

Assumption~\ref{ass.cross} holds
for the two Gaussian AR sources of Section~\ref{sec.exp.ar} (Appendix~\ref{app.ident}).
Conversely, if two Gaussian sources, after rescaling, have the same conditional distribution given their contexts, every rotation of the two sources is again self-sufficient, so the sources are not identifiable (Appendix~\ref{app.ident}).
The theorem requires $\vg$ to be a diffeomorphism that acts on each time point separately and in the same way, which further inspires the projection steps in our algorithms (Section~\ref{sec.learn.demixing}). 

\section{Learning De-mixing Flows}
\label{sec.learn.demixing}

In this section, we propose numerical methods that learn de-mixing flows.

Let $\rmZ = [\rvz_1, \dots, \rvz_T] = \vg(\rmX)$ be \emph{the recovered signals}. The detailed construction of $\vg$ is given below. As in Section~\ref{sec.ssica}, $\rmC_t = \mask_t(\rmZ)$ is the context of $\rvz_t$, with rows $\rvc_{i,t}$. We write
\begin{align*}
  \pi_t(\vz \mid \rmC_t) &:= p(\rvz_t = \vz \mid \rmC_t), \\
  \pi_{i,t}(z_i \mid \rvc_{i,t}) &:= p(\rz_{i,t} = z_i \mid \rvc_{i,t})
\end{align*}
as the joint conditional and the dimension-wise conditionals of $\rmZ$. The SICA assumption in \eqref{eq.sica} requires $\rmZ$ to satisfy $\pi_t(\vz \mid \rmC_t) = \prod_i \pi_{i,t}(z_i \mid \rvc_{i,t})$, which suggests a learning criterion
minimizing the KL divergence between the joint conditional and the product of dimension-wise conditionals:
\begin{align}
    \vg^* &:= \argmin_\vg \cJ(\rmZ), \label{eq.obj.original} \\
    \cJ(\rmZ) &:= \sum_{t\in [T]} \E_{\rmC_t} \KL\Big[ \pi_t(\cdot \mid \rmC_t) \,\Big\Vert\, \prod_{i=1}^{d} \pi_{i,t}(\cdot \mid \rvc_{i,t}) \Big]. \nonumber
\end{align}
This objective enforces the factorization of conditional distributions of the recovered signal $\rmZ$. It is closely related to the total correlation $\TC(\rmZ)$ defined in equation \ref{eq.tc.def}:
\begin{proposition}
\label{prop.tc.relation}
For any recovered signal $\rmZ$,
\[
  \cJ(\rmZ) = \sum_{t=1}^{T} \big[\TC(\rmZ) - \TC(\rmZ_{:,-t})\big],
\]
hence $0 \le \cJ(\rmZ) \le T \cdot \TC(\rmZ)$,
where $\rmZ_{:,-t}$ denotes $\rmZ$ with its $t$-th column removed. Moreover, if $\rmZ$ has a strictly positive density, then $\cJ(\rmZ) = 0$ if and only if $\TC(\rmZ) = 0$.
\end{proposition}
Thus $\cJ$ vanishes exactly when the recovered signals are independent sequences, and it is bounded above by $T\cdot\TC(\rmZ)$. Unlike $\TC$, however, each term of $\cJ$ only involves the conditional distributions at a single time point, which is what makes the objective tractable, as we show below. The proof is given in Appendix~\ref{app.tc}.

Minimizing KL-divergence often leads to adversarial training \citep{nowozin_f-gan_2016},
and adversarial optimization is known to be numerically challenging.

\begin{algorithm}[t]
\caption{Iterative KL minimization Algorithm}
\label{alg:gcd}
\begin{algorithmic}[1]
\Require A mixed multi-dimensional sequence $\rmX$
\Ensure Recovered signals $\rmZ^{(J)}$
\State $\rmZ^{(0)} \gets \rmX$
\For{$j = 1, \dots, J$}
    \State $\vl^{(j)} \gets \argmin_{\vl \in \mathcal{L}} \cJ_{\rmZ^{(j-1)}}(\vl)$ \Comment{refinement}
    \State $\rvz_t^{(j)} \gets \big(\Pi\vl^{(j)}\big)(\rvz_t^{(j-1)})$ for all $t$ \Comment{projection}
\EndFor
\State \Return $\rmZ^{(J)}$
\end{algorithmic}
\end{algorithm}

Instead of directly minimizing the KL divergence in \eqref{eq.obj.original}, we propose an iterative approach to gradually refine $\rvz_t$.
At iteration $j$, given the current iterate $\rmZ^{(j-1)}$ (with $\rmZ^{(0)} = \rmX$),
we learn a refinement function $\vl^{(j)}(\cdot\,; \rmC_t)$ that maps each column $\rvz_t^{(j-1)}$ of the current iterate to a refined column, using the other time points $\rmC_t = \mask_t(\rmZ^{(j-1)})$ as context; its projection (below) then gives $\rvz_t^{(j)}$.
Its goal is to move the conditional distribution of $\rvz_t$ given $\rmC_t$ towards the factorized form $\prod_i \pi_{i,t}(\cdot \mid \rvc_{i,t})$ required by self-sufficiency, so that repeated refinements gradually de-mix the signals.

The refinement $\vl^{(j)}$ is learned by
minimizing the surrogate objective $\vl^{(j)} := \argmin_{\vl \in \mathcal{L}} \cJ_{\rmZ^{(j-1)}}(\vl)$.
To lighten notation, we define the surrogate objective for a generic iterate $\rmZ$, and drop the superscript $(j-1)$ from $\rmZ$, $\rmC_t$ and the conditionals $\pi_t, \pi_{i,t}$; 
\begin{multline}
    \label{eq.kl.iterative}
    \cJ_{\rmZ}(\vl) := \sum_{t\in [T]} \E_{\rmC_t} \KL\Big[ \vl_{\#}\pi_t(\cdot \mid \rmC_t) \\
    \Big\Vert\, \prod_{i=1}^{d} \pi_{i,t}(\cdot \mid \rvc_{i,t}) \Big],
\end{multline}
and $\vl_{\#}\pi_t(\cdot \mid \rmC_t)$ denotes the conditional distribution of $\vl(\rvz_t; \rmC_t)$ given $\rmC_t$. 
One can see that \eqref{eq.kl.iterative} closely resembles \eqref{eq.obj.original},
but the contexts $\rmC_t$ and the target decomposition $\prod_i \pi_{i,t}$, are treated as constants \emph{carried over} from the $(j-1)$-th iteration, thus do not depend on $\vl$.
Restricting the random variables that depend on $\vl$ to just $\rvz_t$ at each iteration is a crucial step toward deriving a tractable, non-adversarial algorithm to minimize the KL divergence in \eqref{eq.obj.original}.

The refinement $\vl^{(j)}$ depends on the context $\rmC_t$, so it may transform the same value of $\rvz_t$ differently at different time points. In the standard ICA setting, however, the de-mixing function is a single map applied to every time point, as the mixing $\rvx_t = \vf(\rvs_t)$ is, and Theorem~\ref{thm.ident} identifies the sources among such de-mixing functions. We therefore project each refinement onto a family $\bar{\mathcal{L}}$ of context-free maps $\R^d \to \R^d$, which gives $\Pi\vl^{(j)} \in \bar{\mathcal{L}}$, and apply it to every time point,
\[
  \rvz_t^{(j)} = \big(\Pi\vl^{(j)}\big)\big(\rvz_t^{(j-1)}\big), \qquad t \in [T].
\]
Since the refinements below are given by velocity fields, the projection $\Pi$, which depends on the current iterate, fits a context-free velocity field in a family $\mathcal{F}$ to their velocities by least squares over all time points, and $\Pi\vl^{(j)}$ is the map that this field generates (one Euler step for SICA-WGF, the flow map for SICA-RF; Sections~\ref{sec.min.kl.wgf} and~\ref{sec.min.kl.rf}).
After $J$ iterations, the algorithm returns the recovered signal $\rmZ^{(J)}$, and the learned de-mixing function is the composite of the projected refinements,
\[
  \vg := \Pi\vl^{(J)} \circ \cdots \circ \Pi\vl^{(1)}, \qquad \rvz_t^{(J)} = \vg(\rvx_t), \quad t \in [T],
\]
a single map applied to every time point, as in the standard ICA setting. Theorem~\ref{thm.ident} requires $\vg$ to be a $C^1$ diffeomorphism of $\R^d$, which holds if every $\Pi\vl^{(j)}$ is.
Algorithm~\ref{alg:gcd} outlines the method; Algorithm~\ref{alg:gcd2} (Appendix~\ref{app.alg}) gives both variants.
Section~\ref{sec.min.kl.wgf} establishes local descent of $\TC(\rmZ)$ for the idealized WGF update under a population projection condition.

\Eqref{eq.obj.original} is an objective over $t$ and we would normally need many samples at the same time step $t$ to minimize the KL at each $t$.
However, in this paper, we adopt the following stationarity assumption.
\begin{assumption}
\label{ass.iden}
At every iteration, for all $i \in [d]$ and $t, t' \in [T]$,
$\pi_t = \pi_{t'}$ and $\pi_{i,t} = \pi_{i,t'}$.
\end{assumption}
Assumption \ref{ass.iden} assumes that the conditionals are identical over the entire time period. This stationarity condition allows us to use samples more efficiently,
since we can reuse samples at every $t$ for the estimation. 
Under Assumption \ref{ass.iden}, we drop the index $t$ in the rest of this section: $(\rvz, \rmC)$ denotes a generic pair $(\rvz_t, \rmC_t)$ with rows $\rvc_i$, and $\pi, \pi_i$ denote the common conditionals.

\subsection{Minimizing KL using Wasserstein Gradient Flow}
\label{sec.min.kl.wgf}
In this and the next section, we focus on Lines 3 and 4 in Algorithm \ref{alg:gcd}. 
The methods apply to every iteration $j$, so, as in \eqref{eq.kl.iterative}, we drop the superscript $(j-1)$ and write $\rmZ$ for the current iterate.

To optimize \eqref{eq.kl.iterative}, we need to choose the family $\mathcal{L}$.
Mixing functions can be viewed as flows of ODEs: e.g., a linear mixing $\rvx = \mA\rvs$, with $\mA$ having a real logarithm $\mV$, is the time-one map of $\mathrm{d}\rvs/\mathrm{d}\tau = \mV \rvs$, and a de-mixing reverses such a flow (Appendix~\ref{sec.mix.demix.flows}).
Inspired by this and by flow-based generative models, we define the refinement function $\vl$ using an ODE.
Let $\vl(\rvz; \rmC)$ be the solution of an ODE using $\rvz$ as the initial value, i.e., $\rvy(0) = \rvz$. The ODE evolves according to the following differential equation
\begin{align}
    \label{eq.ode.vv_t}
    \frac{\mathrm{d}}{\mathrm{d}\tau} \rvy(\tau) = \vv(\rvy(\tau), \rmC, \tau).
\end{align}
Then, learning $\vl$ translates into learning the velocity field $\vv$. If a density $\rho_\tau$ follows the steepest descent of $\KL[\rho_\tau \Vert \nu]$ for a target density $\nu$, i.e., the Wasserstein-2 gradient flow (WGF) of the KL divergence \citep{ambrosio_gradient_2008}, its velocity field is $-\nabla \log (\rho_\tau / \nu)$ \citep[Theorem~1.4.1]{chewi_log_2025}.

In our case, let $\rho_\tau(\cdot \mid \rmC)$ be the conditional density of $\rvy(\tau)$ given $\rmC$, so that $\rho_0 = \pi$, and let the target be $\nu = \prod_{i} \pi_i(\cdot \mid \rvc_i)$.
The WGF velocity field that minimizes our KL divergence is
\begin{align}
  \label{kl.wgf}
  \vv^*(\vy; \rmC) := -\nabla_\vy \log \frac{\rho_\tau(\vy \mid \rmC)}{\prod_i \pi_i(y_i \mid \rvc_i)},
\end{align}
where $\vv^*: \R^{d} \times \tilde{\R}^{d \times T} \to \R^{d}$.
Estimating this conditional density ratio would require conditional samples. However, the conditional ratio and the joint ratio differ only by the factor $\prod_i p(\rvc_i) / p(\rmC)$, which does not depend on $\vy$, so
\begin{align}
    \label{eq.log.joint.ratio}
    \vv^*(\vy; \rmC) = -\nabla_\vy \log \frac{\rho_\tau(\vy, \rmC)}{\prod_i p(\rz_i = y_i, \rvc_i)}.
\end{align}
By Assumption~\ref{ass.iden}, the pairs $\{(\rvz_t, \rmC_t)\}_{t=1}^{T}$ are samples from the numerator at $\tau = 0$. For the denominator, we permute the pairs $\{(\rz_{i,t}, \rvc_{i,t})\}_{t=1}^{T}$ of each row $i$ with an independent random permutation $\sigma_i$ of $[T]$; the permuted pairs $(\rz'_{i,t}, \rvc'_{i,t})$ mimic independent copies of the rows, i.e., samples from $p(\rz_i, \rvc_i)$, independently across $i$.
Any density ratio estimator \citep{sugiyama_density_2012} then estimates the ratio in \eqref{eq.log.joint.ratio}, and auto-differentiation gives its gradient; we use telescoping density-ratio estimation \citep{rhodes_telescoping_2020}, as a single classifier saturates when the signals are strongly mixed (Appendix~\ref{sec.hyper}).

After obtaining $\vv^*$, we solve the ODE in \eqref{eq.ode.vv_t} using the Euler method. In experiments, we observe that even one step of the Euler update can already work effectively, so that the refinement function is $\vl(\rvz; \rmC) = \rvz + \eta\, \vv^*(\rvz; \rmC)$ with step size $\eta$. Its projection is
\[
  (\Pi\vl)(\rvz) = \rvz + \eta\, \vu(\rvz),
\]
where the context-free velocity field $\vu \in \mathcal{F}$ is fitted to the velocities by least squares, $\sum_{t \in [T]} \|\vv^*(\rvz_t; \rmC_t) - \vu(\rvz_t)\|^2$: an affine least-squares fit plus a rank-one non-linear correction fitted to its residuals with early stopping (Appendix~\ref{app.rf.proj}). We refer to this algorithm as SICA-WGF. 
The next theorem gives a population condition for local descent of $\TC(\rmZ)$.

\begin{theorem}
\label{thm.tc.decrease}
Assume that $\rmZ$ has a smooth, strictly positive density, and that $\vu$ fits the velocities at least as well as the zero field, $\sum_{t} \E\|\vv^*(\rvz_t; \rmC_t) - \vu(\rvz_t)\|^2 \le \sum_{t} \E\|\vv^*(\rvz_t; \rmC_t)\|^2$. Then the SICA-WGF step $\rmZ_\eta := [(\Pi\vl)(\rvz_t)]_{t \in [T]}$ satisfies
\[
  \frac{\mathrm{d}}{\mathrm{d}\eta} \TC(\rmZ_\eta) \Big|_{\eta=0} \le -\frac{1}{2} \sum_{t=1}^{T} \E\|\vu(\rvz_t)\|^2 \le 0.
\]
\end{theorem}
The condition on $\vu$ holds for a population least-squares fit over a family containing the zero field. Hence, for sufficiently small $\eta$, every SICA-WGF step decreases $\TC$ unless $\vu = 0$.
This is an idealized full-context result with exact $\vv^*$, not a guarantee for finite-sample, windowed updates. Appendix~\ref{app.tc} gives the proof and an error-dependent bound; see Section~\ref{sec.related} for prior work.

\subsection{Minimizing KL using Rectified Flow}
\label{sec.min.kl.rf}
If we do not require the density to follow the steepest descent, any ODE that transports $\pi(\cdot \mid \rmC)$ to the target $\prod_i \pi_i(\cdot \mid \rvc_i)$ reduces the KL divergence to zero, which allows distribution-matching flows such as Rectified Flow \citep{liu_flow_2022}.
For source and target samples $\rvw_0$ and $\rvw_1$, rectified flow learns the velocity field $\vv^* := \argmin_\vv \E\|\rvw_1 - \rvw_0 - \vv(\rvw_\tau, \tau)\|^2$ of the interpolation $\rvw_\tau := \tau \rvw_1 + (1-\tau) \rvw_0$, with $\tau \sim \mathrm{U}[0,1]$.

For the population construction, let $(\rz'_i,\rvc'_i) \sim p(\rz_i,\rvc_i)$ be mutually independent row copies, independent of $(\rvz,\rmC)$, and stack them as $(\rvz',\rmC')$. The row permutations of Section~\ref{sec.min.kl.wgf} approximate these samples.
We take the source $(\rvz, \rmC, \rmC')$ and the target $(\rvz', \rmC, \rmC')$, i.e., we interpolate only the signal, $\rvz(\tau) := \tau \rvz' + (1-\tau) \rvz$, while keeping both contexts fixed.
With this source and target, the rectified flow objective becomes
\begin{equation}
    \label{eq.rec.obj}
    \vv^* := \argmin_\vv \E\big\|\rvz' - \rvz - \vv(\rvz(\tau), \rmC, \rmC', \tau)\big\|^2,
\end{equation}
where $\tau \sim \mathrm{U}[0,1]$ and $\vv: \R^{d} \times \tilde{\R}^{d \times T} \times \tilde{\R}^{d \times T} \times \R \to \R^{d}$.

Let $\phi_\tau(\cdot\,; \rmC, \rmC')$ be the flow map of $\vv^*$, i.e., the solution of $\frac{\mathrm{d}}{\mathrm{d}\tau}\phi_\tau = \vv^*(\phi_\tau, \rmC, \rmC', \tau)$ with $\phi_0(\vz; \rmC, \rmC') = \vz$. Our refinement function sets the second context to $\rmC$: $\vl(\vz; \rmC) := \phi_1(\vz; \rmC, \rmC)$.
\begin{theorem}
\label{thm.rectified.flow}
    Under the population and regularity conditions of Appendix~\ref{sec.thm.rectified.flow.proof}, $\vl(\cdot\,; \rmC)$ transports $\rvz$ to the target, i.e.,
    $\vl_{\#}\pi(\cdot \mid \rmC) = \prod_i \pi_i(\cdot \mid \rvc_i)$.
\end{theorem}
Hence $\vl$ reduces the KL divergence in \eqref{eq.kl.iterative} to zero (proof in Appendix~\ref{sec.thm.rectified.flow.proof}).
Its projection $\Pi\vl := \psi_1$ is the flow map of a context-free, time-dependent velocity field $\vu_\tau$ (the time-dependent analogue of $\mathcal{F}$), fitted by least squares to $\vv^*(\cdot, \rmC_t, \rmC_t, \tau)$ along the trajectories of the ODE of $\vv^*$ from every $\rvz_t$, i.e., $\frac{\mathrm{d}}{\mathrm{d}\tau}\psi_\tau(\vz) = \vu_\tau(\psi_\tau(\vz))$ with $\psi_0(\vz) = \vz$ (Appendix~\ref{app.rf.proj}). After the projection, the refinement transports $\pi(\cdot \mid \rmC)$ to the target only approximately.

Unlike SICA-WGF, SICA-RF is not a gradient flow of $\TC$.

\section{Experiments}
\label{sec.exp}
We now evaluate the empirical performance of SICA on both artificial and real-world datasets. For both flow variants (WGF and RF), the velocity field $\vv$ is parameterized using a three-layer convolutional neural network (CNN) with 16 hidden channels (\verb|torch.nn.Conv1d| for sequences and \verb|torch.nn.Conv2d| for images).
We use a centered context window of $2k+1$ time points ($k=16$ for AR(7) and speech; a $5\times5$ patch for MNIST). Its center is always masked, aligning the masks after row permutation in Algorithm~\ref{alg:gcd2}; otherwise, mask positions alone could distinguish the sample sets. This local approximation assumes distant values add little predictive information.
We consider several linear and non-linear ICA variants:
FastICA \citep{hyvarinen_fast_1999}, LICA \citep{suzuki_least-squares_2011}, Permutation-Contrastive Learning (PCL) \citep{hyvarinen_nonlinear_2017} and iVAE \citep{khemakhem_variational_nodate}, as well as AMUSE \citep{tong_indeterminacy_1991}, a linear method that separates signals by their temporal correlation.
See the supplementary material for more implementation details and source code to reproduce experimental results. 

\subsection{Autoregressive Signals}
\label{sec.exp.ar}
\begin{figure*}[t]
\centering
\subfigure[AR(7)]{\includegraphics[width=.235\textwidth]{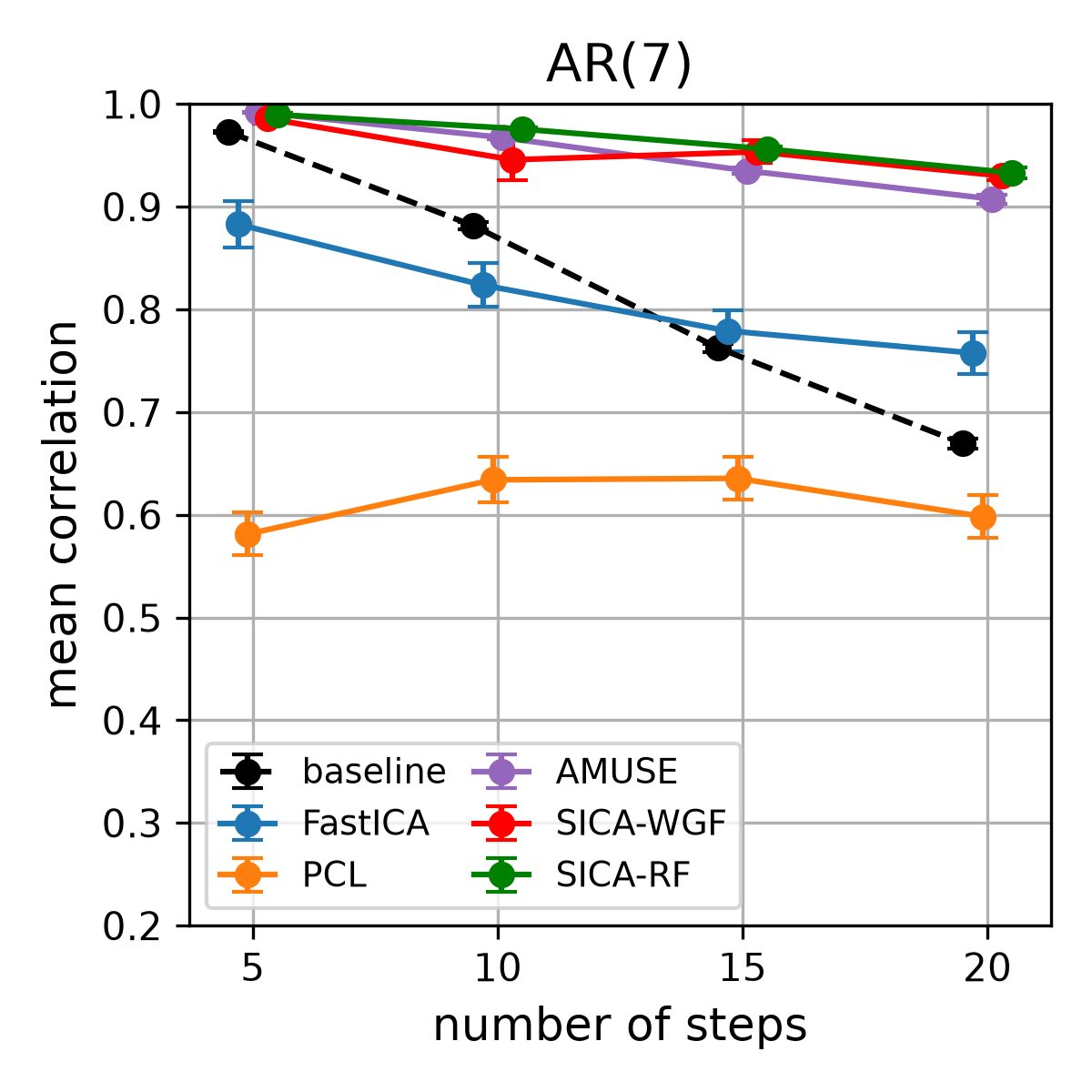}\label{fig.main.ar}}\hfill
\subfigure[MNIST]{\includegraphics[width=.235\textwidth]{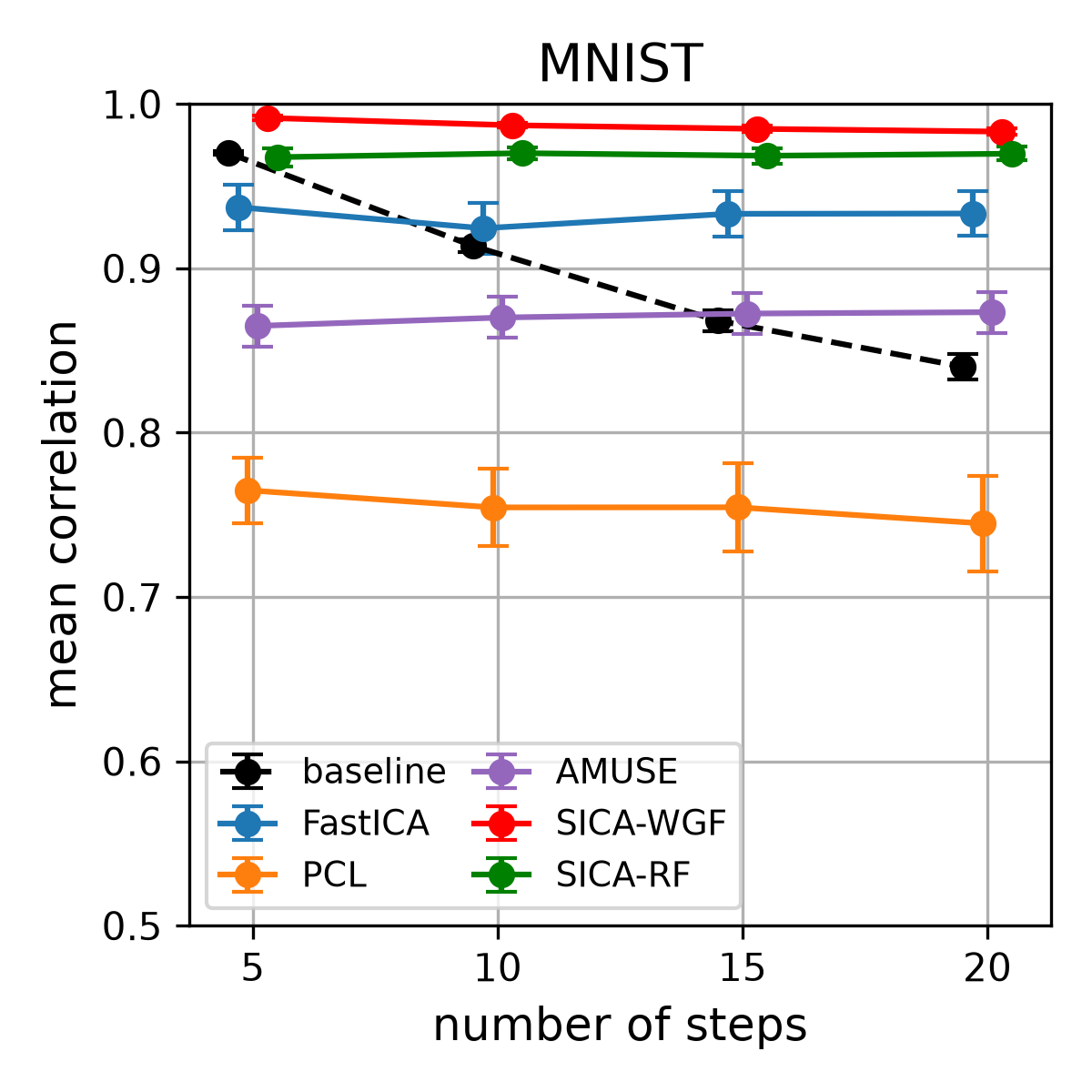}\label{fig.main.mnist}}\hfill
\subfigure[LibriSpeech]{\includegraphics[width=.235\textwidth]{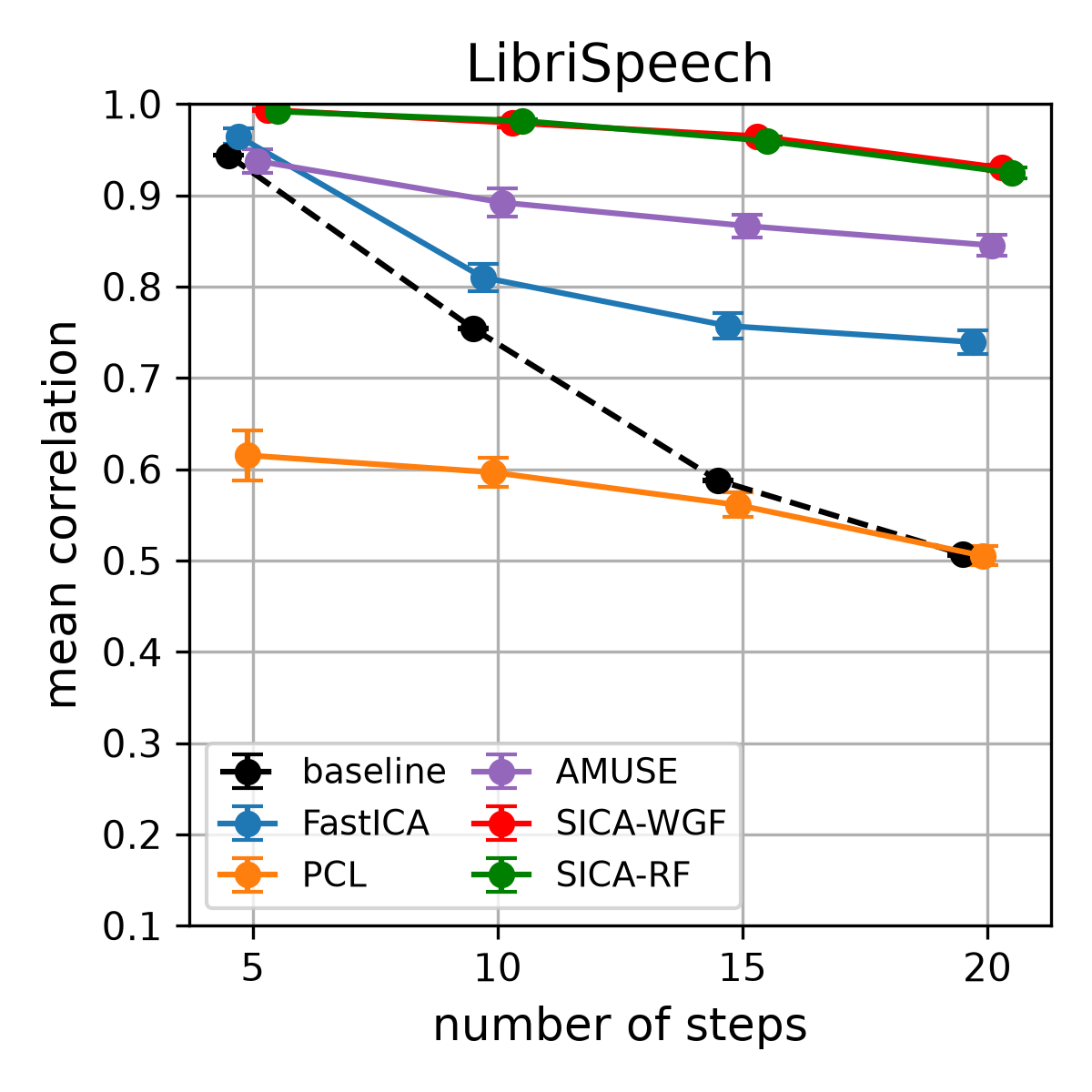}\label{fig.main.libri}}\hfill
\subfigure[CMU ARCTIC]{\includegraphics[width=.235\textwidth]{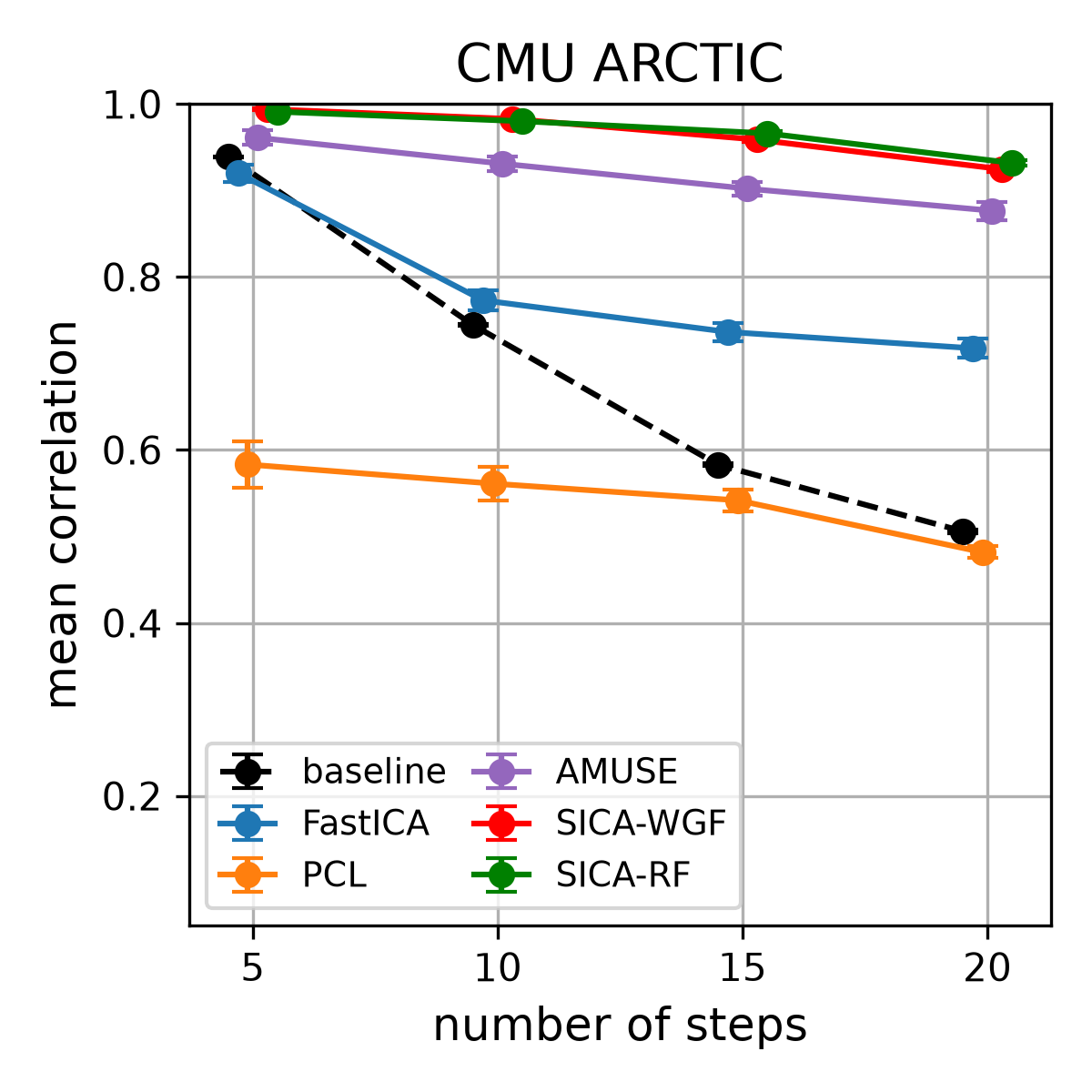}\label{fig.main.arctic}}
\caption{MCC under non-linear mixing $\vf$ over the number of mixing steps $M$ (higher is better), over 20 independent runs; error bars represent the standard error. Results under linear mixing and illustrative examples are in Appendix~\ref{app.exp.figs}.}
\label{fig.main.mcc}
\end{figure*}

In this experiment, the true signals are generated independently using autoregressive models, and are created using the following formula:
\begin{gather*}
    \rs_{i, t} = \textstyle\sum_{k=1}^7 \beta_{i,k} \rs_{i, t-k} + \epsilon, \quad \epsilon \sim \mathcal{N}(0, 0.1^2), \\
    \boldsymbol{\beta}_1 = 0.1 \times [4, 3, 2, 1, -0.5, 0.3, -0.2], \\
    \beta_{2,k} = (-1)^k \beta_{1,k}.
\end{gather*}
The second signal thus has the spectrum of the first, mirrored about a quarter the sampling rate. The two signals need different coefficients: two Gaussian AR processes with the same coefficients have the same autocovariance, so every rotation of them is again a pair of independent AR sequences, and the sources are identifiable only up to a rotation (Appendix~\ref{app.ident}).
$\rs_{i, 1} \dots \rs_{i, 7}$ are sampled from the standard normal distribution, and the mixing function $\vf$ is constructed as an iterative update on the true signal. At step $m$, we mix the signal with an update
$\rvx^{(m)} = \rvx^{(m-1)} + 0.1\,\vh(\mW\rvx^{(m-1)})$, where $\mW \in \mathbb{R}^{d \times d}$ and $\vh: \mathbb{R}^{d} \to \mathbb{R}^{d}$.
We initialize $\rvx^{(0)}$ as $\rvs$. 
When $\vh$ is the identity function, the mixed signal $\rvx$ is a linear function of $\rvs$, given by 
$
\rvx = (\mI + 0.1\,\mW)^{M} \rvs,
$
where $M$ counts mixing updates, distinct from the $J$ de-mixing iterations. In this section, we consider the setting where the number of signals is $d = 2$ and the sequence length is $T = 1024$.
$\mW$ is specified with diagonal entries equal to $1$ and off-diagonal entries equal to $0.7$. 
We examine two scenarios: (i) $\vh$ as the identity function, which will be referred to as the \emph{linear} setting, and 
(ii) $\vh$ as the GELU activation, which will be referred to as the \emph{non-linear} setting.

First, we perform 20 non-linear mixing updates and use SICA with RF to recover the original signal. We then illustrate the recovered signal in Figure~\ref{fig.mcc.demo} (Appendix~\ref{app.exp.figs}), where we can see that the observed signals (blue) are almost visually indistinguishable due to the mixing function. However, the signal recovered by SICA (red) closely traces the true signal (black) (MCC 0.88, from 0.68 for the observed signals).

We measure mean correlation coefficients (MCC): the mean absolute Pearson correlation after optimal one-to-one matching of recovered and true signals. We plot MCC over various numbers of mixing steps ($M$). The larger $M$ is, the more mixed the signals are.

In the linear setting (Figure~\ref{fig.mcc.linear}, Appendix~\ref{app.exp.figs}), FastICA and LICA stay at an MCC of 0.87--0.90 at all mixing steps: the sources are Gaussian, so methods based on non-Gaussianity or on the instantaneous distribution cannot fully separate them. Methods that use the temporal structure can: AMUSE recovers the sources exactly, and both SICA methods reach an MCC of at least 0.98 at all mixing steps (the baseline is the MCC of the observed signal $\rvx$).

We now demonstrate the advantage of our method using the non-linear setting. In this setting, we also compare SICA with non-linear ICA variants (PCL and iVAE).
Figure~\ref{fig.main.ar} shows that MCCs of all methods decrease as the number of mixing steps increases. AMUSE is the strongest baseline. At $M = 5$, where the mixing is mild, it is slightly ahead of SICA (0.992, vs 0.990 for SICA-RF and 0.986 for SICA-WGF). As the mixing gets stronger, SICA-RF leads from $M = 10$ and SICA-WGF from $M = 15$, and the lead grows with the mixing (at $M = 20$: 0.933 and 0.930 vs 0.908), since AMUSE is limited to linear de-mixing. FastICA and PCL stay far below; iVAE does not achieve an MCC above 0.2 at any mixing step, so it is not plotted for better visualization of the other methods.

\subsection{MNIST Images}
In this set of experiments, we evaluate the performance of ICA methods on the task of de-mixing overlaid images. We randomly select samples from the MNIST dataset \citep{lecun_gradient-based_1998}, zero-padded to $32\times32$ pixels, and treat each image as a signal indexed by its pixel positions, $\rvs \in \mathbb{R}^{d \times 32\times32}$ ($d=2$ images in the MCC experiments). We then apply the linear and non-linear mixing steps described in the previous section to obtain mixed signals $\rvx$.

On an illustrative example with three digits, SICA-WGF recovers the digits and keeps their polarity (white on black), whereas FastICA leaves one channel mixed with flipped strokes (Appendix~\ref{app.exp.figs}).

We further evaluate SICA using MCC over multiple runs (Figures~\ref{fig.main.mnist} and~\ref{fig.mnist.linear}). In the linear setting, SICA-WGF leads at all mixing steps (MCC 0.98--0.99), and SICA-RF (0.97) outperforms the linear ICA methods (0.94--0.95) at all mixing steps. AMUSE stays at 0.86: the correlation between neighboring pixels is similar for different digits.
In the non-linear setting (Figure~\ref{fig.main.mnist}), we compare SICA with two non-linear ICA variants (PCL and iVAE). iVAE stays below an MCC of 0.45 (not plotted). It can be seen that SICA-WGF again leads at all mixing steps, including $M = 5$, where the mixing is mild (0.99, vs 0.97 for the baseline). SICA-RF is second for $M \geq 10$ (0.97); at $M = 5$, it is on par with the baseline.

\subsection{Speech Signals}
We also separate mixtures of real speech recordings from two corpora: Mini LibriSpeech, the dev-clean-2 subset of LibriSpeech \citep{panayotov_librispeech_2015}, and CMU ARCTIC \citep{kominek_cmu_2004}, of which we use the speakers bdl, slt, clb and rms. We take the first 2 seconds of every utterance that is at least 2 seconds long, resample it to 1 kHz ($T = 2000$), scale it to $[-1, 1]$ and add Gaussian noise with standard deviation 0.2. In each run, we select $d = 3$ utterances at random (for CMU ARCTIC, from three different speakers) and mix them with the linear and non-linear mixing steps of Section~\ref{sec.exp.ar}; the context window is the same as for AR(7).

Figures~\ref{fig.main.libri} and~\ref{fig.main.arctic} (non-linear mixing) and Figure~\ref{fig.speech} in Appendix~\ref{app.exp.figs} (both mixings) show the results, which are consistent across the two corpora. In the linear setting, both SICA methods recover the sources almost perfectly (MCC 0.99 or higher) at all mixing steps. FastICA is also strong (MCC 0.97--0.99), as speech signals are highly non-Gaussian, AMUSE reaches 0.96--0.98, and LICA stays around 0.8. In the non-linear setting, the two SICA methods lead at all mixing steps and are on par with each other. For the strongest mixing ($M = 20$), they reach an MCC of 0.925--0.932 on both corpora, compared with 0.85--0.88 for AMUSE, the strongest baseline, and 0.72--0.74 for FastICA. PCL stays at or below the baseline.

\section{Related Work}
\label{sec.related}

Although derived from an iterative KL minimization (Algorithm~\ref{alg:gcd}), SICA-WGF meets a classic idea. Without the context, its velocity $-\nabla \log\big(p(\rvz_t) / \prod_i p(\rz_{i,t})\big)$ is the negative score function difference (SFD) of \citet{babaie-zadeh_differential_2004}, the nonparametric gradient of the mutual information at a single time point, which ignores temporal dependence. With the full context, the stacked SICA-WGF velocities coincide with the negative SFD of whole sequences, i.e., the gradient of $\TC(\rmZ)$ (Lemma~\ref{lem.velocity}), which explains why SICA-WGF decreases $\TC$ (Theorem~\ref{thm.tc.decrease}). For whole sequences, however, their approach would need the SFD of $\rmZ$, i.e., the scores of its $dT$-dimensional density, of which a single sequence provides only one sample, and \citet{babaie-zadeh_general_2005} apply it only to single time points or pairs of lagged outputs. SICA instead uses conditional density ratios at each time point, estimated from the $T$ context pairs under stationarity. Full-context pairs still have dimension $dT$. Moreover, \citet{babaie-zadeh_general_2005} target mixing models that independence alone makes separable, whereas self-sufficiency identifies the sources under more general non-linear mixing (Theorem~\ref{thm.ident}). Finally, they refit a fixed parametric separating function to the updated signals at every iteration, whereas SICA composes the projected steps, $\vg = \Pi\vl^{(J)} \circ \cdots \circ \Pi\vl^{(1)}$, which is far more expressive than that family; SICA-RF has no counterpart there.

The cross-derivative equation in the proof of Theorem~\ref{thm.ident} also appears in PCL \citep[Theorem~1]{hyvarinen_nonlinear_2017} and its extensions \citep{halva_disentangling_2021,schell_nonlinear_2023}. The closest identifiability condition is that of \citet[Definition~7 and Theorem~3]{schell_nonlinear_2023}, which uses the bivariate marginal densities of the sources instead of their conditional densities.

Identifiability has also been obtained from latent clusters, without side information \citep{willetts_dont_2021} or with instance labels \citep{reizinger_cross-entropy_2025}, and from exchangeable, non-i.i.d.\ data \citep{reizinger_identifiable_2025}; SICA assumes neither and instead uses the context of each source as its auxiliary variable (Theorem~\ref{thm.ident}).

\section{Conclusions}
We proposed SICA, an ICA criterion under which each source is best predicted from its own context, and learned it with projected de-mixing flows.
Our approach has several limitations. Identifiability requires a de-mixing that acts on every time point in the same way, and sources that depend on their own contexts in sufficiently different ways; Gaussian sources with the same conditional structure are not identifiable. The algorithm assumes stationarity and retrains a density-ratio estimator or a flow at every iteration, and our experiments involve at most three sources. Finally, the projected iterates may stop before reaching a SICA solution.

\subsection*{AI use statement}

In this work, we used generative AI tools to formulate mathematical claims, to provide critical ingredients for proving mathematical claims and to assist in the writing of proofs, in particular the proof of Theorem~\ref{thm.ident}, the descent result of Theorem~\ref{thm.tc.decrease}, the transport result of Theorem~\ref{thm.rectified.flow}, and the invertibility criterion in Appendix~\ref{app.rf.proj}. We also used them to design and provide feedback on the methodology and experiments, for example the projection step of Algorithm~\ref{alg:gcd} and the control experiments; to implement methods and experiments; and to interpret results.

We have not used generative AI tools to develop the theoretical model or conceptual framework, namely the self-sufficiency criterion and the de-mixing flows, or to propose or refine hypotheses. Assistance with translation and qualitative or thematic data analysis are not applicable to this work.

Additionally, we used generative AI tools to draft and edit parts of the paper, suggest its structure, identify and summarize relevant literature, format references, and create and edit code and figures.

We have reviewed all AI-assisted work. All AI-assisted claims and proofs were checked by the authors. All statements about prior work were checked against the original papers.

We take responsibility for the final content of this work, including text, claims, or artifacts produced with the aid of generative AI.

\bibliography{main}

\clearpage
\appendix
\thispagestyle{empty}

\onecolumn
\aistatstitle{Self-sufficient Independent Component Analysis for Demixing Flows: \\
Supplementary Materials}

\section{Missing Proofs}
\subsection{Proof of Equation~\ref{eq.aica}}
\label{app:proof_aica}
\begin{proof}
By the chain rule for conditional densities with a fixed ordering $1,\dots,d$, we have
\begin{equation}
    p(\rvs \mid \rmU) \;=\; \prod_{i=1}^d p\!\left(\rs_i \,\middle|\, \rvs_{1:i-1},\, \rmU\right).
\end{equation}
By the conditional independence assumption in Equation~\ref{eq.aica.independnece},
\begin{equation}
    \rs_i \ind (\rvs_{-i},\rvu_{-i}) \,\big|\, \rvu_i,
\end{equation}
we have
\begin{equation}
    p\!\left(\rs_i \,\middle|\, \rvs_{1:i-1},\, \rmU\right) \;=\; p(\rs_i \mid \rvu_i).
\end{equation}
Substituting this back into the chain rule gives
\begin{equation}
    p(\rvs \mid \rmU) \;=\; \prod_{i=1}^d p(\rs_i \mid \rvu_i),
\end{equation}
as required.

The other direction is also true. 
Fix $i$. Using Bayes' rule and \eqref{eq.aica},
\begin{align*}
p(\rs_i \mid \rvs_{-i}, \rmU)
&= \frac{p(\rs_i, \rvs_{-i} \mid \rmU)}{p(\rvs_{-i}\mid \rmU)}
 = \frac{p(\rvs\mid \rmU)}{\int p(\rvs\mid \rmU)\, \mathrm{d}s_i} \\
&= \frac{ \big[p(\rs_i\mid \rvu_i)\prod_{j\neq i} p(\rs_j\mid \rvu_j)\big] }{
      \prod_{j\neq i} p(\rs_j\mid \rvu_j) \;\int p(\rs_i\mid \rvu_i)\, \mathrm{d}s_i } \\
&= \frac{ p(\rs_i\mid \rvu_i)}{1}
 = p(\rs_i\mid \rvu_i),
\end{align*}
since $\int p(\rs_i\mid \rvu_i)\,\mathrm{d}s_i=1$ (replace the integral by a sum in the discrete case).
Hence $\rs_i \ind (\rvs_{-i},\rvu_{-i}) \mid \rvu_i$. 

\end{proof}

\subsection{Proofs of Proposition \ref{prop.tc.relation} and Theorem \ref{thm.tc.decrease}}
\label{app.tc}

Throughout, $\rmZ \in \R^{d \times T}$ has a smooth, strictly positive density $p$, and $p_i$ denotes the marginal density of the $i$-th row $\rmZ_{i,:}$. As random variables, the context $\rmC_t = \mask_t(\rmZ)$ is equivalent to $\rmZ_{:,-t}$, and its $i$-th row $\rvc_{i,t}$ is equivalent to $\rmZ_{i,-t}$. We write
\[
  \vs(\rmZ) := \nabla_{\rmZ} \log p(\rmZ) \in \R^{d \times T},
  \qquad
  \big[\vs_{\mathrm{row}}(\rmZ)\big]_{i,:} := \nabla_{\rmZ_{i,:}} \log p_i(\rmZ_{i,:}).
\]
Up to sign, $\vs_{\mathrm{row}} - \vs$ is the score function difference of \citet{babaie-zadeh_differential_2004}, with the rows of $\rmZ$ playing the role of the components.
We write $\mathcal{V}(\rmZ) := [\vv^*(\rvz_t; \rmC_t)]_{t \in [T]}$ and $\bar{\mathcal{V}}(\rmZ) := [\vu(\rvz_t)]_{t \in [T]}$ for the velocities before and after the projection, stacked over time, so that the condition of Theorem~\ref{thm.tc.decrease} reads $\E\|\mathcal{V} - \bar{\mathcal{V}}\|_F^2 \le \E\|\mathcal{V}\|_F^2$ and its bound reads $-\frac{1}{2}\E\|\bar{\mathcal{V}}\|_F^2$.

\begin{proof}[Proof of Proposition \ref{prop.tc.relation}]
By definition,
\[
  \cJ_t(\rmZ) := \E\Big[\log p(\rvz_t \mid \rmZ_{:,-t}) - \sum_{i} \log p(\rz_{i,t} \mid \rmZ_{i,-t})\Big]
  = \sum_{i=1}^{d} H(\rz_{i,t} \mid \rmZ_{i,-t}) - H(\rvz_t \mid \rmZ_{:,-t}).
\]
Writing each conditional entropy as a difference of joint entropies, and using that $(\rz_{i,t}, \rmZ_{i,-t})$ is $\rmZ_{i,:}$ and $(\rvz_t, \rmZ_{:,-t})$ is $\rmZ$,
\[
  \cJ_t(\rmZ) = \sum_{i=1}^{d} \big[H(\rmZ_{i,:}) - H(\rmZ_{i,-t})\big] - \big[H(\rmZ) - H(\rmZ_{:,-t})\big]
  = \TC(\rmZ) - \TC(\rmZ_{:,-t}).
\]
Summing over $t$ gives the identity. Since $\cJ_t \ge 0$ (a KL divergence) and $\TC(\rmZ_{:,-t}) \ge 0$, we obtain $0 \le \cJ(\rmZ) \le T \cdot \TC(\rmZ)$.

For the equivalence, if $\TC(\rmZ) = 0$, the rows are mutually independent and $p(\rvz_t \mid \rmC_t) = \prod_i p(\rz_{i,t} \mid \rvc_{i,t})$, so $\cJ(\rmZ) = 0$. Conversely, $\cJ_t(\rmZ) = 0$ implies $\rz_{i,t} \ind \rmZ_{-i,:} \mid \rmZ_{i,-t}$ for all $i$ and $t$. Fix $i$ and let $\rmR := \rmZ_{-i,:}$. Since $p$ is strictly positive, the intersection property of conditional independence applies \citep{lauritzen_graphical_1996}: from $\rmR \ind \rz_{i,1} \mid (\rz_{i,2}, \dots, \rz_{i,T})$ and $\rmR \ind \rz_{i,2} \mid (\rz_{i,1}, \rz_{i,3}, \dots, \rz_{i,T})$, we obtain $\rmR \ind (\rz_{i,1}, \rz_{i,2}) \mid (\rz_{i,3}, \dots, \rz_{i,T})$. Repeating this argument for $t = 3, \dots, T$ yields $\rmR \ind \rmZ_{i,:}$. As this holds for every $i$, the rows are mutually independent, i.e., $\TC(\rmZ) = 0$.
\end{proof}

\begin{lemma}
\label{lem.velocity}
For every $t \in [T]$, the WGF velocity field \eqref{kl.wgf} at $\tau = 0$ satisfies $\vv^*(\rvz_t; \rmC_t) = \big[\vs_{\mathrm{row}}(\rmZ) - \vs(\rmZ)\big]_{:,t}$. Hence $\mathcal{V}(\rmZ) = \vs_{\mathrm{row}}(\rmZ) - \vs(\rmZ)$.
\end{lemma}
\begin{proof}
Since $p(\rvz_t \mid \rmC_t) = p(\rmZ) / p(\rmZ_{:,-t})$ and the denominator does not depend on $\rvz_t$, we have $\nabla_{\rvz_t} \log p(\rvz_t \mid \rmC_t) = [\vs(\rmZ)]_{:,t}$. Similarly, $p(\rz_{i,t} \mid \rvc_{i,t}) = p_i(\rmZ_{i,:}) / p(\rmZ_{i,-t})$, so $\partial_{\rz_{i,t}} \log p(\rz_{i,t} \mid \rvc_{i,t}) = [\vs_{\mathrm{row}}(\rmZ)]_{i,t}$. Substituting both identities into \eqref{kl.wgf} gives the claim.
\end{proof}

\begin{lemma}[First variation of entropy]
\label{lem.entropy}
Let $\rvx_\eta = \rvx + \eta\, \rvw$, where $\rvw$ is jointly distributed with $\rvx$, and assume that the boundary terms below vanish. Then
\[
  \frac{\mathrm{d}}{\mathrm{d}\eta} H(\rvx_\eta) \Big|_{\eta=0} = -\,\E\big[\rvw^\top \nabla \log p_{\rvx}(\rvx)\big].
\]
\end{lemma}
\begin{proof}
Let $\bar{\vw}(\vx) := \E[\rvw \mid \rvx = \vx]$. The density of $\rvx_\eta$ satisfies the continuity equation $\partial_\eta p = -\nabla \cdot (p\, \bar{\vw})$ at $\eta = 0$ \citep[Lemma 1]{babaie-zadeh_differential_2004}. Therefore, using $\int \partial_\eta p = 0$, integration by parts, and the tower property,
\begin{align*}
  \frac{\mathrm{d}}{\mathrm{d}\eta} H
  &= -\int \partial_\eta p \, \log p \,\mathrm{d}\vx
  = \int \nabla \cdot (p\, \bar{\vw}) \log p \,\mathrm{d}\vx \\
  &= -\int p\, \bar{\vw}^\top \nabla \log p \,\mathrm{d}\vx
  = -\,\E\big[\rvw^\top \nabla \log p_{\rvx}(\rvx)\big].
\end{align*}
\end{proof}

\begin{proof}[Proof of Theorem \ref{thm.tc.decrease}]
Let $\widetilde{\mathcal{V}}(\rmZ)$ be either $\bar{\mathcal{V}}(\rmZ)$ (with the projection) or $\mathcal{V}(\rmZ)$ (without it), and let $\rmZ_\eta := \rmZ + \eta\, \widetilde{\mathcal{V}}(\rmZ)$ denote the refined signal.
Write $\TC(\rmZ) = \sum_i H(\rmZ_{i,:}) - H(\rmZ)$. Applying Lemma~\ref{lem.entropy} to $\rmZ$ with $\rvw = \widetilde{\mathcal{V}}$, and to each row $\rmZ_{i,:}$ with $\rvw = \widetilde{\mathcal{V}}_{i,:}$,
\begin{align*}
  \frac{\mathrm{d}}{\mathrm{d}\eta} \TC(\rmZ_\eta) \Big|_{\eta=0}
  &= -\sum_{i=1}^{d} \E\big[\widetilde{\mathcal{V}}_{i,:}^\top \nabla \log p_i(\rmZ_{i,:})\big] + \E\big[\langle \widetilde{\mathcal{V}}, \vs(\rmZ) \rangle\big] \\
  &= -\,\E\big[\langle \widetilde{\mathcal{V}},\ \vs_{\mathrm{row}}(\rmZ) - \vs(\rmZ) \rangle\big]
  = -\,\E \big\langle \widetilde{\mathcal{V}}(\rmZ), \mathcal{V}(\rmZ) \big\rangle,
\end{align*}
where the last equality uses Lemma~\ref{lem.velocity}. For $\widetilde{\mathcal{V}} = \bar{\mathcal{V}}$, the condition $\E\|\mathcal{V} - \bar{\mathcal{V}}\|_F^2 \le \E\|\mathcal{V}\|_F^2$ is equivalent to $2\,\E\langle \mathcal{V}, \bar{\mathcal{V}} \rangle \ge \E\|\bar{\mathcal{V}}\|_F^2$, which gives the claim.
\end{proof}

\paragraph{Remarks.}
(i) If $\vu$ minimizes $\E\|\mathcal{V} - \bar{\mathcal{V}}\|_F^2$ over a family that contains the zero field, the condition of Theorem~\ref{thm.tc.decrease} holds. If this family is also closed under scaling, as $\mathcal{F}$ is, then $c \mapsto \E\|\mathcal{V} - c\,\bar{\mathcal{V}}\|_F^2$ is minimal at $c = 1$, so $\E\langle \mathcal{V} - \bar{\mathcal{V}}, \bar{\mathcal{V}} \rangle = 0$, and the rate equals $-\E\|\bar{\mathcal{V}}\|_F^2$.
(ii) The rate can vanish before $\rmZ$ satisfies SICA, namely when $\mathcal{V}(\rmZ)$ is orthogonal to all stacked fields $[\vu'(\rvz_t)]_{t \in [T]}$, $\vu' \in \mathcal{F}$.
(iii) Without the projection ($\widetilde{\mathcal{V}} = \mathcal{V}$), the rate is $-\E\|\mathcal{V}(\rmZ)\|_F^2$; for scalar components instead of rows, this is the result of \citet[Theorem 2]{babaie-zadeh_differential_2004} and \citet[Eq.~(47)]{babaie-zadeh_general_2005}.
If the support of $\rmZ$ is connected, this rate vanishes if and only if $\rmZ$ satisfies SICA: it vanishes if and only if $\mathcal{V} = 0$ almost surely, i.e., $\nabla \log \big(p / \prod_i p_i\big) = 0$, so $p = c \prod_i p_i$ for a constant $c$, and $c = 1$ because both sides integrate to one; hence $\TC(\rmZ) = 0$, which by Proposition~\ref{prop.tc.relation} is equivalent to $\cJ(\rmZ) = 0$.
(iv) The first variation of $\TC$ with respect to $p$ is $\log p - \sum_i \log p_i$, whose negative gradient is $\vs_{\mathrm{row}} - \vs = \mathcal{V}$; thus $\mathcal{V}$ is the negative Wasserstein gradient of $\TC$.

\paragraph{Estimated velocities and approximate projection.}
The same first-variation identity quantifies the gap between the population result and an estimated update. Let $\widehat{\mathcal{V}}$ be an estimated stacked velocity and $\bar{\mathcal{V}}=[\vu(\rvz_t)]_{t\in[T]}$. Treat the fitted fields as fixed and take expectations over a fresh sequence, under the regularity conditions used above. Define $\varepsilon^2:=\E\|\widehat{\mathcal{V}}-\mathcal{V}\|_F^2$ and $r^2:=\E\|\bar{\mathcal{V}}\|_F^2$, and suppose that, for some $\delta\geq0$,
\[
 \E\|\widehat{\mathcal{V}}-\bar{\mathcal{V}}\|_F^2
 \leq \E\|\widehat{\mathcal{V}}\|_F^2+\delta.
\]
Expanding this inequality and applying Cauchy--Schwarz gives
\begin{align*}
 \left.\frac{\mathrm{d}}{\mathrm{d}\eta}\TC(\rmZ+\eta\bar{\mathcal{V}})\right|_{\eta=0}
 &=-\E\langle\mathcal{V},\bar{\mathcal{V}}\rangle\\
 &=-\E\langle\widehat{\mathcal{V}},\bar{\mathcal{V}}\rangle
   +\E\langle\widehat{\mathcal{V}}-\mathcal{V},\bar{\mathcal{V}}\rangle
 \leq -\tfrac12 r^2+\varepsilon r+\tfrac12\delta.
\end{align*}
Thus a nonzero update is a strict descent direction whenever $r^2>2\varepsilon r+\delta$. Here $\varepsilon$ includes both velocity-estimation error and error from replacing the full context by a window; $\delta$ measures a population fitting shortfall. Training or held-out error improvements alone do not bound these quantities, so this result is a sufficient condition, not a finite-sample guarantee for our implementation.

\subsection{Proof of Theorem \ref{thm.rectified.flow}}
\label{sec.thm.rectified.flow.proof}

\paragraph{Population sampling and regularity conditions.}
The theorem concerns exact population regression, with the independent row copies of Section~\ref{sec.min.kl.rf}, not finite permutations of one observed sequence. Assume $\E(\|\rvz\|^2+\|\rvz'\|^2)<\infty$. Omit the fixed masked entries when viewing contexts as Euclidean vectors. Write $\mathcal{A}:=\operatorname{supp}p(\rmC)$ and $\mathcal{B}:=\prod_i\operatorname{supp}p(\rvc_i)$. The following are sufficient regularity conditions:
\begin{enumerate}
\item The conditional laws $\pi(\cdot\mid\mC)$ on $\mathcal{A}$ and $\pi_i(\cdot\mid\vc_i)$ on each marginal context support admit weakly continuous versions.
\item The population minimizer $\vv^*$ admits a jointly continuous version in state, both contexts, and $\tau\in[0,1]$. It is locally Lipschitz in the state, uniformly on compact sets of states and contexts, and has at most linear growth in the state, uniformly for contexts in compact sets and $\tau\in[0,1]$.
\end{enumerate}
These conditions give a unique global flow on $[0,1]$ that depends continuously on its initial state and contexts. They also give uniqueness of the continuity equation in the class of probability curves with finite first moments used below. Independence gives context-pair support $\mathcal{A}\times\mathcal{B}$, and $\mathcal{A}\subseteq\mathcal{B}$, so every diagonal pair $(\mC,\mC)$ with $\mC\in\mathcal{A}$ lies in this support. Continuity is essential: population regression determines a velocity only almost everywhere, whereas the diagonal can have probability zero.

\begin{proof}
    Let $\rvw := \phi_1(\rvz; \rmC, \rmC')$ be the output of the flow map defined in Section \ref{sec.min.kl.rf}.
    Due to Lemma \ref{lem2.2},
    \begin{align}
    \label{eq.fac1}
      p(\rvw \mid \rmC, \rmC') = p(\rvz' \mid \rmC, \rmC').
    \end{align}

    Since $(\rvz',\rmC')$ is independent of $(\rvz,\rmC)$ in the population construction,
    \begin{align}
    \label{eq.fac2}
      p(\rvz' \mid \rmC, \rmC') = p(\rvz' \mid \rmC')
      = \prod_i \pi_i(\cdot \mid \rvc'_i),
    \end{align}
    where the last equality holds because the row copies $(\rz'_i, \rvc'_i) \sim p(\rz_i, \rvc_i)$ are mutually independent across $i$ by the construction in Section \ref{sec.min.kl.rf}.
    Combining \eqref{eq.fac1} and the factorization in \eqref{eq.fac2}, we get
    \begin{align}
    \label{eq.factor}
        p(\rvw \mid \rmC, \rmC') = \prod_i \pi_i(\cdot \mid \rvc'_i).
    \end{align}
    Since $\rvz$ is independent of $\rmC'$ given $\rmC$, \eqref{eq.factor} is an equality of the measures $(\phi_1(\cdot;\mC,\mC'))_{\#}\pi(\cdot\mid\mC)$ and $\prod_i\pi_i(\cdot\mid\vc'_i)$ for almost every context pair. Both sides are weakly continuous in $(\mC,\mC')$: for the left side, use weak continuity of the source law and joint continuity of the flow. A set of full probability is dense in the support $\mathcal{A}\times\mathcal{B}$; testing against bounded continuous functions therefore extends this equality to every pair in that support. In particular, it holds at $(\mC,\mC)$, giving
    \[
      \vl_{\#}\pi(\cdot \mid \rmC) = \big(\phi_1(\cdot\,; \rmC, \rmC)\big)_{\#}\pi(\cdot \mid \rmC) = \prod_i \pi_i(\cdot \mid \rvc_i),
    \]
    which is the desired result.
\end{proof}

\begin{lemma}
    \label{lem2.2}
    Under the conditions above, $p\big(\phi_1(\rvz; \rmC, \rmC') \,\big|\, \rmC, \rmC'\big) = p(\rvz' \mid \rmC, \rmC')$ for almost every context pair.
\end{lemma}
\begin{proof}
    This proof is a conditional version of the Marginal Preserving Theorem (Theorem 3.3) of \citet{liu_flow_2022}.
    Let $q_\tau(\vz \mid \rmC, \rmC')$ be the conditional density function of the interpolation $\rvz(\tau)$.
    First, we can see that for any $h(\cdot)$ that vanishes at the boundary of the domain,
    \begin{align}
        \label{eq.lemma.e1}
        \partial_\tau \E[h(\rvz(\tau)) \mid \rmC, \rmC'] = \int h(\vz)\, \partial_\tau q_\tau(\vz \mid \rmC, \rmC') \,\mathrm{d}\vz.
    \end{align}
    Second, since the interpolation $\rvz(\tau)$ is a deterministic function of $\rvz(0) = \rvz$ and $\rvz(1) = \rvz'$, we can see that
    {\setlength{\jot}{7pt}
    \begin{align}
        \partial_\tau \E[h(\rvz(\tau)) \mid \rmC, \rmC']
        &\eqby{reparam.\ trick} \E[\partial_\tau h(\rvz(\tau)) \mid \rmC, \rmC'] \notag \\
        &\eqby{LOTUS} \E[\nabla^\top h(\rvz(\tau))\, \partial_\tau \rvz(\tau) \mid \rmC, \rmC'] \notag \\
        &= \E\big[\nabla^\top h(\rvz(\tau))\, \E[\partial_\tau \rvz(\tau) \mid \rvz(\tau), \rmC, \rmC'] \,\big|\, \rmC, \rmC'\big] \notag \\
        &= \E\big[\nabla^\top h(\rvz(\tau))\, \E[\rvz' - \rvz \mid \rvz(\tau), \rmC, \rmC'] \,\big|\, \rmC, \rmC'\big] \notag \\
        &\eqby{optimal sol.} \E[\nabla^\top h(\rvz(\tau))\, \vv^*(\rvz(\tau), \rmC, \rmC', \tau) \mid \rmC, \rmC'] \label{eq.optimal.sol} \\
        &= \int \nabla^\top h(\vz)\, \vv^*(\vz, \rmC, \rmC', \tau)\, q_\tau(\vz \mid \rmC, \rmC') \,\mathrm{d}\vz \notag \\
        &\eqby{int.\ by parts} - \int h(\vz)\, \nabla \cdot \big[\vv^*(\vz, \rmC, \rmC', \tau)\, q_\tau(\vz \mid \rmC, \rmC')\big] \,\mathrm{d}\vz, \label{eq.lemma.e2}
    \end{align}}%
    where LOTUS stands for the law of the unconscious statistician, and \eqref{eq.optimal.sol} is due to the fact that the conditional expectation $\E[\rvz' - \rvz \mid \rvz(\tau), \rmC, \rmC']$ is the optimal solution of the training objective \eqref{eq.rec.obj}.
    From both \eqref{eq.lemma.e1} and \eqref{eq.lemma.e2}, we can see that the continuity equation holds for $\vv^*$:
    \begin{align}
        \partial_\tau q_\tau(\vz \mid \rmC, \rmC') = - \nabla \cdot \big[\vv^*(\vz, \rmC, \rmC', \tau)\, q_\tau(\vz \mid \rmC, \rmC')\big], \quad \text{a.s.}
    \end{align}
    These identities are understood weakly, using smooth compactly supported test functions. For almost every fixed context pair, the interpolation law and the flow-pushforward law solve the same continuity equation with initial law $\pi(\cdot\mid\rmC)$. The local Lipschitz and growth conditions above ensure uniqueness, so the two laws agree for all $\tau\in[0,1]$. The extension to diagonal context pairs is established separately in the proof of the theorem.
    The fact that $q_1 = p(\rvz' \mid \rmC, \rmC')$ concludes the proof.
\end{proof}

\subsection{Proof of Theorem~\ref{thm.ident}}
\label{app.ident}

Throughout, $t$ is fixed and $t'$ ranges over the other time points. Let $\vk := (\vg \circ \vf)^{-1}$, so that $\rvs_{t'} = \vk(\rvz_{t'})$ for every $t'$, and write $\partial_j k_i := \partial k_i / \partial z_j$, which is the $(i,j)$ entry of the Jacobian $\mJ_\vk$. For a matrix of context values, $\vk$ is applied to each column (the NaN column is left unchanged). We regard $\log p(\rs_{i,t} \mid \rvu_{i,t})$ as a function of $\rs_{i,t}$ and of the entries $\rs_{i,t'}$, $t' \neq t$, of the context, and write $\partial / \partial \rs_{i,t}$ and $\partial / \partial \rs_{i,t'}$ for its partial derivatives; likewise for $\log p(\rz_{i,t} \mid \rvc_{i,t})$, with $\rvc_{i,t} = (\rz_{i,t'})_{t' \neq t}$. For $\vs$ and $\vs'$ as in Assumption~\ref{ass.cross}, let $\mU^{(n)}$ be the value of the context $\rmU_t$ whose $t_n$-th column is $\vs'$ and whose other entries take the values of the assumption, so that $\omega_{i,n}$ is the cross-derivative $\partial^2 \log p(\rs_{i,t} \mid \rvu_{i,t}) / \partial \rs_{i,t}\, \partial \rs_{i,t_n}$ at $\rs_{i,t} = s_i$ and $\rvu_{i,t}$ equal to the $i$-th row of $\mU^{(n)}$. The joint densities of $(\rvs_t, \rmU_t)$ and of $(\rvz_t, \rmC_t)$ are positive and continuous, and we use the continuous versions of the conditional densities. The factorization in the hypothesis of the theorem holds almost everywhere; since $\log p(\rvz_t \mid \rmC_t)$ is continuous and equals a sum of functions of the disjoint blocks $(\rz_{i,t}, \rvc_{i,t})$ almost everywhere, it is such a sum everywhere, with continuous summands, and we take these summands as the versions of $\log p(\rz_{i,t} \mid \rvc_{i,t})$.

\begin{lemma}[Change of variables]
\label{lem.cov}
For all values of $\rvz_t$ and $\rmC_t$,
\[
\log p(\rvz_t \mid \rmC_t) = \log p\bigl(\vk(\rvz_t) \mid \vk(\rmC_t)\bigr) + \log\bigl|\det \mJ_\vk(\rvz_t)\bigr| .
\]
\end{lemma}

\begin{proof}
The map from $(\rvs_t, \rmU_t)$ to $(\rvz_t, \rmC_t)$ applies the diffeomorphism $\vg \circ \vf$ to each column, so its inverse has a block-diagonal Jacobian with blocks $\mJ_\vk(\rvz_t)$ and $\mJ_\vk(\rvz_{t'})$, $t' \neq t$. Hence $p(\rvz_t, \rmC_t) = p(\vk(\rvz_t), \vk(\rmC_t))\, |\det \mJ_\vk(\rvz_t)| \prod_{t' \neq t} |\det \mJ_\vk(\rvz_{t'})|$, and, in the same way, $p(\rmC_t) = p(\vk(\rmC_t)) \prod_{t' \neq t} |\det \mJ_\vk(\rvz_{t'})|$. Dividing the two gives the claim.
\end{proof}

\noindent
The Jacobian term depends on $\rvz_t$ only. With the SICA factorizations of $\rmZ$ (the hypothesis of the theorem) and of $\rmS$ (\eqref{eq.sica}), Lemma~\ref{lem.cov} gives, for all values of $\rvz_t$ and $\rmC_t$,
\begin{equation}
\label{eq.ident.basic}
\sum_{i=1}^d \log p(\rz_{i,t} \mid \rvc_{i,t}) = \sum_{i=1}^d \log p(\rs_{i,t} \mid \rvu_{i,t}) + L(\rvz_t), \qquad L := \log|\det \mJ_\vk| ,
\end{equation}
where, on the right-hand side, $\rs_{i,t'} = k_i(\rvz_{t'})$ for every $t'$, including $t' = t$.

\begin{lemma}[Joint diagonalization]
\label{lem.jd}
Let $\mQ, \mQ' \in \R^{d \times d}$ be invertible, and let $\boldsymbol{\omega}_1, \dots, \boldsymbol{\omega}_d \in \R^N$, with entries $\omega_{i,n}$, be nonzero and pairwise linearly independent. If $\mQ^\top \diag(\omega_{1,n}, \dots, \omega_{d,n})\, \mQ'$ is diagonal for every $n \in [N]$, then $\mQ$ is a permutation matrix times an invertible diagonal matrix.
\end{lemma}

\begin{proof}
Let $\mD_n := \diag(\omega_{1,n}, \dots, \omega_{d,n})$ and $m_{j,n} := (\mQ^\top \mD_n \mQ')_{jj}$. Transposing, $\mQ'^\top \mD_n \mQ$ is diagonal with the same diagonal entries, and, as $\mQ'$ is invertible, $\mD_n \mQ = \mQ'^{-\top}(\mQ'^\top \mD_n \mQ)$, i.e.
\[
\omega_{i,n}\, Q_{ij} = (\mQ'^{-\top})_{ij}\, m_{j,n} \qquad \text{for all } i, j, n .
\]
Fix $j$ and let $\vm_j := (m_{j,1}, \dots, m_{j,N})$. Every $i$ with $Q_{ij} \neq 0$ has $\boldsymbol{\omega}_i = \bigl((\mQ'^{-\top})_{ij} / Q_{ij}\bigr)\, \vm_j$. If the $j$-th column of $\mQ$ had nonzero entries in two rows $i \neq i'$, then $\boldsymbol{\omega}_i$ and $\boldsymbol{\omega}_{i'}$ would both be multiples of $\vm_j$ and, being nonzero, parallel. Hence every column of $\mQ$ has at most one nonzero entry, and, as $\mQ$ is invertible, exactly one, in distinct rows.
\end{proof}

\noindent
This is the uniqueness of non-orthogonal joint diagonalization \citep{leurgans_decomposition_1993}. For AMUSE, with $\rvx_t = \mA\rvs_t$ and a de-mixing matrix $\mB$, one has $\mB \operatorname{Cov}(\rvx_t, \rvx_{t+\tau}) \mB^\top = \mQ^\top \operatorname{Cov}(\rvs_t, \rvs_{t+\tau}) \mQ$ with $\mQ = (\mB\mA)^\top$, the $\boldsymbol{\omega}_i$ collect the autocovariances of the sources at lags $0$ and $\tau$, and their pairwise linear independence is the condition that the autocorrelations at lag $\tau$ are distinct \citep{tong_indeterminacy_1991}. Lemma~\ref{lem.jd} is also verified in Lean~4 with Mathlib (file \texttt{lean/JointDiag.lean} in the supplementary material), in the equivalent form that every row and every column of $\mQ$ has at most one nonzero entry; the proof uses only the standard axioms.

\begin{proof}[Proof of Theorem~\ref{thm.ident}]
\emph{Step 1: a diagonal congruence from the cross-derivatives.}
Fix $t' \neq t$ and $j \neq l$. Only the $l$-th term on the left-hand side of \eqref{eq.ident.basic} contains $\rz_{l,t'}$, as the $t'$-th entry of its context. On the right-hand side, $L(\rvz_t)$ does not contain it, and in the $i$-th term it enters only through the context entry $\rs_{i,t'} = k_i(\rvz_{t'})$. The right-hand side is differentiable in $\rz_{l,t'}$, hence so is $\log p(\rz_{l,t} \mid \rvc_{l,t})$, and the chain rule gives
\[
\frac{\partial \log p(\rz_{l,t} \mid \rvc_{l,t})}{\partial \rz_{l,t'}} = \sum_{i=1}^d \frac{\partial \log p(\rs_{i,t} \mid \rvu_{i,t})}{\partial \rs_{i,t'}}\, \partial_l k_i(\rvz_{t'}),
\]
where $\rvu_{i,t} = (k_i(\rvz_{t''}))_{t'' \neq t}$ is the $i$-th row of $\rmU_t = \vk(\rmC_t)$. The left-hand side does not depend on $\rz_{j,t}$, since $j \neq l$. On the right-hand side, $\partial_l k_i(\rvz_{t'})$ and $\rvu_{i,t}$ do not depend on $\rvz_t$, and $\rs_{i,t} = k_i(\rvz_t)$. Differentiating in $\rz_{j,t}$ gives, for all values of $\rvz_t$ and $\rmC_t$,
\begin{equation}
\label{eq.ident.cross}
\sum_{i=1}^d \frac{\partial^2 \log p(\rs_{i,t} \mid \rvu_{i,t})}{\partial \rs_{i,t}\, \partial \rs_{i,t'}}\, \partial_j k_i(\rvz_t)\, \partial_l k_i(\rvz_{t'}) = 0 \qquad (j \neq l),
\end{equation}
that is, $\mJ_\vk(\rvz_t)^\top \mG_{t'}\, \mJ_\vk(\rvz_{t'})$ is diagonal, where $\rvs_t = \vk(\rvz_t)$ and $\mG_{t'}$ is the diagonal matrix of these cross-derivatives, $i \in [d]$. The SICA factorization is used on both sides: on the left, for the $l$-th term alone to contain $\rz_{l,t'}$; on the right, for $\mG_{t'}$ to be diagonal. Only first derivatives of $\vk$ appear, and $L$ is never differentiated.

\emph{Step 2: the Jacobian is a scaled permutation.}
Let $\vs \in \R^d$, with $\vs'$, $t_n$ and $\mU^{(n)}$ as in Assumption~\ref{ass.cross} and above, and let $\vz := \vk^{-1}(\vs)$ and $\vz' := \vk^{-1}(\vs')$. For each $n$, evaluate \eqref{eq.ident.cross} with $t' = t_n$, $\rvz_t = \vz$ and $\rmC_t = \vk^{-1}(\mU^{(n)})$. Then $\rmU_t = \mU^{(n)}$ and $\rvz_{t_n} = \vz'$, so that, with $\mQ := \mJ_\vk(\vz)$ and $\mQ' := \mJ_\vk(\vz')$, which are the same for all $n$ and invertible, the matrices $\mQ^\top \mG_{t_n} \mQ' = \mQ^\top \diag(\omega_{1,n}, \dots, \omega_{d,n})\, \mQ'$ are diagonal for $n \in [N]$. By Lemma~\ref{lem.jd}, $\mJ_\vk(\vz)$ is a permutation matrix times an invertible diagonal matrix; in particular, $\partial_j k_i(\vz)\, \partial_l k_i(\vz) = 0$ for all $i$ and $j \neq l$. As $\vs$ is arbitrary, this holds at every $\vz \in \R^d$ (if the assumption held only for $\vs$ in a dense set, the corresponding $\vz$ would form a dense set, and continuity of $\mJ_\vk$ would give the same), so $\partial_j k_i\, \partial_l k_i = 0$ on $\R^d$ for all $i$ and $j \neq l$.

\emph{Step 3: each $k_i$ depends on a single coordinate.}
Fix $i$ and let $\Omega_l := \{\vz \in \R^d : \partial_l k_i(\vz) \neq 0\}$ for $l \in [d]$. These sets are open, since $\mJ_\vk$ is continuous; pairwise disjoint, by Step 2; and they cover $\R^d$, since the $i$-th row of the invertible matrix $\mJ_\vk(\vz)$ is nonzero. As $\R^d$ is connected, exactly one of them is nonempty; since they cover $\R^d$, it is $\R^d$, and we call its index $\sigma(i)$. Hence $\partial_l k_i \equiv 0$ for $l \neq \sigma(i)$, and, as $\R^d$ is convex, $k_i(\vz) = \kappa_i(z_{\sigma(i)})$ for a $C^1$ function $\kappa_i$ with $\kappa_i' \neq 0$. The map $\sigma$ is a permutation: if $\sigma(i) = \sigma(i')$ for $i \neq i'$, two rows of $\mJ_\vk(\vz)$ would be multiples of the same unit vector, and $\mJ_\vk(\vz)$ would be singular. Since $\vk$ is onto, each $\kappa_i$ is onto $\R$; as $\kappa_i' \neq 0$, it is a bijection and, by the inverse function theorem, a diffeomorphism. As $\rvs_{t'} = \vk(\rvz_{t'})$ for every $t'$, this is the claim.
\end{proof}

\paragraph{Autoregressive sources.}
Let the sources be independent stationary AR(1) processes $\rs_{i,t} = \beta_i \rs_{i,t-1} + \epsilon_{i,t}$ with $0 < |\beta_i| < 1$ and i.i.d.\ innovations with density $e^{\ell_i}$, and let $t$ be an interior time point. By the Markov property, $p(\rs_{i,t} \mid \rvu_{i,t}) \propto \exp\bigl(\ell_i(\rs_{i,t} - \beta_i \rs_{i,t-1}) + \ell_i(\rs_{i,t+1} - \beta_i \rs_{i,t})\bigr)$, with a normalizer that does not depend on $\rs_{i,t}$, so that
\[
\frac{\partial^2 \log p(\rs_{i,t} \mid \rvu_{i,t})}{\partial \rs_{i,t}\, \partial \rs_{i,t-1}} = -\beta_i\, \ell_i''(\rs_{i,t} - \beta_i \rs_{i,t-1}), \qquad \frac{\partial^2 \log p(\rs_{i,t} \mid \rvu_{i,t})}{\partial \rs_{i,t}\, \partial \rs_{i,t+1}} = -\beta_i\, \ell_i''(\rs_{i,t+1} - \beta_i \rs_{i,t}),
\]
and the other cross-derivatives vanish. Take $N = 2$, $t_1 = t - 1$ and $t_2 = t + 1$. For $\vs$ and $\vs'$, the vector $\boldsymbol{\omega}_i$ is proportional to $\bigl(\ell_i''(s_i - \beta_i s'_i),\, \ell_i''(s'_i - \beta_i s_i)\bigr)$. For logistic innovations, $\ell_i'' < 0$, and the ratio of the second to the first entry equals $1$ at $s'_i = s_i$ and tends to $0$ as $s'_i \to \infty$, so, by continuity, it takes every value in $(0, 1]$. Choosing $s'_1, \dots, s'_d$ one after another, we can make the ratios pairwise distinct, i.e.\ the $\boldsymbol{\omega}_i$ pairwise linearly independent, even when all $\beta_i$ are equal: Assumption~\ref{ass.cross} holds. For Gaussian innovations, $\ell_i''$ is constant, every $\boldsymbol{\omega}_i$ is a multiple of $(1, 1)$, and the assumption fails.

\paragraph{Gaussian sources.}
Let the sources be independent, zero-mean Gaussian processes with a positive definite covariance. Then $\rs_{i,t} \mid \rvu_{i,t} \sim \mathcal{N}(\boldsymbol{\gamma}_i^\top \rvu_{i,t}, \nu_i^2)$, where $\boldsymbol{\gamma}_i$ are the interpolation coefficients and $\nu_i > 0$, so the cross-derivative $\partial^2 \log p(\rs_{i,t} \mid \rvu_{i,t}) / \partial \rs_{i,t}\, \partial \rs_{i,t'} = \gamma_{i,t'} / \nu_i^2$ is constant. Taking all $t' \neq t$ as the $t_n$, Assumption~\ref{ass.cross} holds if the $\boldsymbol{\gamma}_i$ are nonzero and pairwise non-proportional, and conversely every $\boldsymbol{\omega}_i$ is a sub-vector (with repetitions) of $\boldsymbol{\gamma}_i / \nu_i^2$, so the assumption fails otherwise. In terms of the precision matrix $\boldsymbol{\Theta}_i$ of the $i$-th source sequence, $\gamma_{i,t'}/\nu_i^2 = -[\boldsymbol{\Theta}_i]_{t,t'}$, so the assumption asks that the rows of the precision matrices at $t$, without their diagonal entries, are pairwise non-proportional. For a stationary source and the full context of a long sequence, the interpolation coefficients are approximately $\gamma_{i,t+k} = -c_{i,k}/c_{i,0}$, where $c_{i,k}$ are the Fourier coefficients of $1/f_i$ and $f_i$ is the spectral density; equal interpolation coefficients then amount to proportional spectral densities, the non-identifiable case of second-order methods \citep{tong_indeterminacy_1991}. The assumption is sufficient, but not necessary: interpolation coefficients that are proportional but not equal are not covered. For the two AR(7) sources of Section~\ref{sec.exp.ar}, with coefficient vectors $\boldsymbol{\beta}_1$ and $\boldsymbol{\beta}_2$, the interpolation coefficients are supported on $|t' - t| \leq 7$ for $7 < t \leq T - 7$, irrespective of the distribution of the first seven values, which does not involve $\rs_{i,t}$; so they are the same for the full sequence and for the window of our experiments, and the cosine of the angle between them is $-0.18$, so they are not parallel. For stationary Gaussian processes, which are time-reversible, the interpolation coefficients from a context window symmetric about $t$ satisfy $\gamma_{i,t-\delta} = \gamma_{i,t+\delta}$, so with the window $\{t \pm 1\}$ the interpolation coefficients of any two such sources are parallel.

If, conversely, $\boldsymbol{\gamma}_1 = \boldsymbol{\gamma}_2 =: \boldsymbol{\gamma}$, the sources are not identifiable. Let $\rvy_{t'} := (\rs_{1,t'}/\nu_1, \rs_{2,t'}/\nu_2)$ and $\rve_t := \rvy_t - \rmY_t \boldsymbol{\gamma}$, where $\rmY_t$ is the $2 \times (T-1)$ matrix with rows $(\rs_{i,t'}/\nu_i)_{t' \neq t}$, $i = 1, 2$. The residuals of the best linear predictors are standard normal and independent of $\rmU_t$, as the variables are jointly Gaussian. For an orthogonal $\mO \in \R^{2 \times 2}$, let $\vg$ map $\rvx_{t'}$ to $(\mO \rvy_{t'}, \rs_{3,t'}, \dots, \rs_{d,t'})$ at every $t'$. Then $\rvz_{1:2,t} = \mO\rmY_t\boldsymbol{\gamma} + \mO\rve_t$, where $\mO\rmY_t$ consists of the first two rows of $\rmC_t$ without its masked column, and $\mO\rve_t \sim \mathcal{N}(\mathbf{0}, \mI_2)$ is independent of $\rmC_t$, so $p(\rvz_t \mid \rmC_t)$ factorizes as in \eqref{eq.sica} for every $\mO$, although $\rvz_{1:2,t}$ mixes the two sources unless $\mO$ is a signed permutation. This is the rotational indeterminacy of second-order methods \citep{tong_indeterminacy_1991}.

\paragraph{Remarks.}
(i) The proof goes through verbatim when the contexts of $\rmS$ and $\rmZ$ are restricted to a window around $t$ that contains $t_1, \dots, t_N$, as in our experiments, with the conditional densities and \eqref{eq.sica} defined with that window. When the $t_n$ differ, Step 2 uses that $\vg$ is the same map at $t_1, \dots, t_N$; if a single time point suffices, it is enough that $\rvz_{t'}$ depends on $\rvx_{t'}$ alone for every $t'$, and the conclusion then holds at time $t$. (ii) A de-mixing in which $\rvz_t$ depends on neighboring observations, such as the refinements of Section~\ref{sec.learn.demixing} without the context-free projection, is not covered: then $\rmC_t$ is not a column-wise transform of $\rmU_t$, and Lemma~\ref{lem.cov} does not apply. (iii) \citet[Theorem~1]{hyvarinen_nonlinear_2019} use the within-time Hessian instead, with the observed context as auxiliary variable; their argument applies to SICA as well and requires that, for every $i$ and $a$, the points $\bigl(\partial \log p(\rs_{i,t} \mid \rvu_{i,t}) / \partial \rs_{i,t},\ \partial^2 \log p(\rs_{i,t} \mid \rvu_{i,t}) / \partial \rs_{i,t}^2\bigr)$, evaluated at $\rs_{i,t} = a$ and all values of $\rvu_{i,t}$, do not lie on a line. Neither condition implies the other: with the one-sided window $\{t - 1\}$, all $\boldsymbol{\omega}_i$ are multiples of $(1, \dots, 1)$ and Assumption~\ref{ass.cross} fails, whereas their condition holds for AR(1) sources with logistic innovations; conversely, their condition fails for Gaussian sources, which Theorem~\ref{thm.ident} covers when the $\boldsymbol{\gamma}_i$ are pairwise non-proportional. (iv) Differences in which context entries the signals depend on suffice: if $\rs_{i,t}$ is not conditionally independent of $\rs_{i,t'}$ given the rest of its context, but $\rs_{j,t}$ is, an entry at $t'$ at which the cross-derivative of signal $i$ is nonzero, together with one at which that of signal $j$ is nonzero, makes $\boldsymbol{\omega}_i$ and $\boldsymbol{\omega}_j$ non-proportional.

\section{Experimental Settings}

\subsection{Neural Network Architecture}
\label{sec.nn.struct}

For both SICA-RF and SICA-WGF, we use a three-layer Convolutional Neural Network (CNN): for SICA-WGF it models one log density ratio of the telescoping estimator of \eqref{eq.log.joint.ratio} (one network per bridge, Appendix~\ref{sec.hyper}), whose gradient gives $\vv$, and for SICA-RF it models the velocity field $\vv$. Since a CNN cannot directly handle NaN values, we use masks to indicate the missing column in the context. As described in Section~\ref{sec.exp}, the context is restricted to a window of $W = 2k+1$ time points centered at $t$. For a $d$-dimensional sequence $\mZ$, the inputs to our neural network are three $d \times W$ matrices: $[\mM_t, \mZ_t, \mC_t]$,
\begin{itemize}
    \item $\mM_t$ is a mask matrix whose elements are all zeros except the center column (the masked time point $t$), which is set to ones.
    \item $\mZ_t$ is a matrix whose columns are $\vz_t$ repeated $W$ times.
    \item $\mC_t$ is the window of $\mZ$ around $t$, with its center column replaced with an arbitrary constant (e.g., zero).
\end{itemize}
These inputs are fed to the CNN block as a $3d$-channel sequence of length $W$.

The CNN block of our neural network contains three layers, each with 16 hidden channels, followed by a linear layer. For SICA-RF, it projects the output from the CNN block to a $\mathbb{R}^d$ vector, as the output of $\vv$; for SICA-WGF, to a scalar, the log density ratio.

Using \verb|torch.nn| PyTorch package, the neural network is defined as 

\begin{verbatim}
        self.net = nn.Sequential(
            nn.Conv1d(3*d, 16, kernel_size=3, padding=1),
            nn.ReLU(),
            nn.Conv1d(16, 16, kernel_size=3, padding=1),
            nn.ReLU(),
            nn.Conv1d(16, 16, kernel_size=3, padding=1), 
            nn.Flatten(start_dim=1)
        )

        self.fc = nn.Linear(16 * W, d)   # SICA-WGF: nn.Linear(16 * W, 1)
\end{verbatim}

Note that for SICA-RF, we need to add another $d$ channels representing the window of $\rmC'_t$ and a constant time channel to the network to reflect that $\vv$ is a time-varying function. SICA-WGF uses SiLU instead of ReLU activations, so that the gradient of the log density ratio is smooth.

For images, the window is a $W \times W$ patch centered at the masked pixel ($W = 5$ for MNIST), the inputs are $d \times W \times W$ tensors, and \verb|nn.Conv1d| is replaced by \verb|nn.Conv2d|.

\subsection{Projection onto Context-free Refinements}
\label{app.rf.proj}

This section gives the details of the projection step of Algorithm~\ref{alg:gcd} (Section~\ref{sec.learn.demixing}). Besides giving a de-mixing function that is applied to every time point, the projection matters for SICA-RF in particular. Theorem~\ref{thm.rectified.flow} holds with the contexts held fixed. The target $\prod_i \pi_i(\cdot \mid \rvc_i)$ is centered at the prediction of each $\rz_i$ from its own context $\rvc_i$, so $\phi_1(\cdot\,; \rmC, \rmC)$ keeps the part of $\rvz_t$ that is predictable from its own context and only makes the unpredictable remainder independent across signals. When the contexts are mixed, the predictable part therefore stays mixed, and repeating the unprojected refinement makes little progress. A context-free map has to transform the predictable part and the remainder of $\rvz_t$ in the same way.

For SICA-RF, we integrate the ODE of $\vv^*(\cdot, \rmC_t, \rmC_t, \tau)$ from every $\rvz_t$, record the states $\vy_{t,\tau}$ and velocities $\vv_{t,\tau} := \vv^*(\vy_{t,\tau}, \rmC_t, \rmC_t, \tau)$ along the trajectories, and fit the context-free field
\begin{align}
\label{eq.rf.proj}
    \vu_\tau(\vy) = \mP_\tau \vy + \vq_\tau + \mB\, \sigma\big(\mA [\vy; \tau] + \ve\big), \qquad
    (\mP_\tau, \vq_\tau) := \argmin_{\mP, \vq} \sum_{t \in [T]} \big\| \vv_{t,\tau} - \mP \vy_{t,\tau} - \vq \big\|^2,
\end{align}
where $\sigma$ is the SiLU activation, $\mB \in \R^{d \times k}$, $\mA \in \R^{k \times (d+1)}$ and $\ve \in \R^k$. The affine part is fitted by least squares at each $\tau$. The correction, shared across $\tau$, is fitted by gradient descent to the residuals $\vv_{t,\tau} - \mP_\tau \vy_{t,\tau} - \vq_\tau$; its Jacobian in $\vy$ has rank at most $k$. The projected refinement is the flow map of the context-free field, $(\Pi\vl)(\vz) := \psi_1(\vz)$ with $\frac{\mathrm{d}}{\mathrm{d}\tau}\psi_\tau(\vz) = \vu_\tau(\psi_\tau(\vz))$ and $\psi_0(\vz) = \vz$, which is invertible and identical for every $t$. For SICA-WGF, which takes a single Euler step, we fit the same field without $\tau$ to the velocities $\vv^*(\rvz_t; \rmC_t)$, and the projected refinement is $(\Pi\vl)(\vz) = \vz + \eta\,\vu(\vz)$.

The capacity of the field matters. An affine field ($\mB = \mathbf{0}$) cannot undo a non-linear mixing, since the composition of affine maps is affine. A flexible field, such as an MLP, can reproduce the context-dependent map: every $\rvz_t$ is a distinct point, so a map $\rvz_t \mapsto \psi_1(\rvz_t)$ can follow the context-specific velocity of each time point. The rank-$k$ correction adds non-linearity in $k$ directions only, and we guard it against fitting context-specific velocities: $\mB$ is initialized at zero, and the correction is trained on 80\% of the time points and stopped when its error on the other 20\% stops decreasing. The held-out time points are taken in contiguous blocks, since neighboring time points have nearly identical states and contexts. If the correction does not reduce the held-out error, it is set to zero, and the field is affine. The affine part is fitted on all time points and does at least as well as the zero field on those estimated velocities. The correction is retained only if it improves the held-out fit. These empirical checks motivate the projection but do not establish the population condition of Theorem~\ref{thm.tc.decrease}: the fitted velocities are estimated, the contexts are windowed, and neighboring samples are dependent. Appendix~\ref{app.tc} separates velocity error from population fitting error. The projected refinement no longer reduces \eqref{eq.kl.iterative} to zero in a single step, but in our experiments we observed that it prevents the recovered signals from deteriorating over the iterations.

\paragraph{Invertibility of the SICA-WGF step.}
Theorem~\ref{thm.ident} requires the learned de-mixing function to be a diffeomorphism. For SICA-RF, the projected refinement is a flow map, which is a diffeomorphism. For SICA-WGF, the projected refinement is the single step $\vy \mapsto \vy + \eta\,\vu(\vy)$, which may fold the space if $\eta$ is large. With $k = 1$, write $\mB = \vb \in \R^d$ and let $\va \in \R^d$ be the first $d$ entries of $\mA$, so that $\vu(\vy) = \mP\vy + \vq + \vb\,\sigma(\va^\top\vy + e)$. By the matrix determinant lemma, the Jacobian determinant of the step is
\[
\det(\mI + \eta\mP)\,\big(1 + \eta\,\sigma'(\va^\top\vy + e)\,\xi\big), \qquad \xi := \va^\top(\mI + \eta\mP)^{-1}\vb,
\]
where $\sigma'$ takes values in $[\sigma'_{\min}, \sigma'_{\max}] \approx [-0.100, 1.100]$. If $\det(\mI + \eta\mP) \neq 0$ and $1 + \eta s\xi$ has the same sign at $s = \sigma'_{\min}$ and $s = \sigma'_{\max}$, the Jacobian is nonsingular everywhere. The matrices $\mI + \eta\mP + \eta\theta\,\vb\va^\top$, $\theta \in [0, 1]$, are then also invertible, and, as $\sigma(x) = \theta x$ with $\theta \in (0, 1)$ the logistic sigmoid of $x$, the step grows at least linearly at infinity. By the global inverse function theorem \citep{gordon_diffeomorphisms_1972}, the step is then a diffeomorphism of $\R^d$. This condition is much weaker than $\eta\,\mathrm{Lip}(\vu) < 1$: a large Lipschitz constant may only mean that the field stretches the signals apart. We checked it at every iteration of SICA-WGF on all four data sets, with $M \in \{5, 20\}$, linear and non-linear mixing and three seeds each, and on three runs of the heart example (510 steps in total): it held at every step, with $|\det(\mI + \eta\mP)| \geq 0.75$ and $|1 + \eta s\xi| \geq 0.67$ at both ends of the range of $s$, whereas $\eta\,\mathrm{Lip}(\vu)$ reached $9.8$.

\subsection{Hyperparameters of the Proposed Method}
\label{sec.hyper}

In addition to the neural network model, which we detailed in Section \ref{sec.nn.struct}, the proposed method has only a few other parameters. In both variants, the network is re-initialized at every iteration of Algorithm~\ref{alg:gcd2}.

For the SICA-WGF method, a single classifier between $\mathcal{D}$ and $\mathcal{D}'_1, \dots, \mathcal{D}'_d$ (Algorithm~\ref{alg:gcd2}) saturates when the signals are strongly mixed, as the two sample sets are then easy to tell apart, and the gradient of its log ratio carries little information. We therefore estimate the log density ratio in \eqref{eq.log.joint.ratio} with telescoping density-ratio estimation \citep{rhodes_telescoping_2020}, which chains classifiers between intermediate distributions, with $m = 4$ bridges. For a joint sample $\vx_0$ from $\mathcal{D}$ and a sample $\vx_m$ from $\mathcal{D}'_1, \dots, \mathcal{D}'_d$ (signal values and contexts), the waymarks are $\vx_k = \sqrt{1 - \alpha_k^2}\, \vx_0 + \alpha_k \vx_m$ with $\alpha_k = k/m$, and the log ratio is the sum $\sum_{k=0}^{m-1} f_k$ of $m$ logistic classifiers, where $f_k$ separates $\vx_k$ from $\vx_{k+1}$. Each classifier is a network as in Appendix~\ref{sec.nn.struct}, and its loss includes the gradient penalty $\lambda\, \E\|\nabla_\vy f_k\|^2$ with $\lambda = 0.01$. The classifiers are trained jointly using the Adam optimizer with a learning rate of 0.001, with a batch size of 100, for 10 epochs. The step size of the Euler update is $\eta = 0.5$ for AR(7) and the speech data, and $\eta = 0.2$ for MNIST and the heart example.

For the SICA-RF method, we train the rectified flow using the Adagrad optimizer with a learning rate of 0.001, with a batch size of 100, and run for 100 epochs. In the sampling stage, we set the number of Euler steps to be 100.
The affine part of the projection \eqref{eq.rf.proj} is fitted separately at each of the 100 Euler steps, and its flow is integrated with the same steps.

For the context-free projection of both variants, we use a correction of rank $k = 1$. It is trained with Adam (learning rate 0.001, batch size 4096) for at most 2000 steps; every 50 steps, the error on the held-out time points (20\%, in blocks of 128 consecutive positions) is evaluated, and training stops after five evaluations without improvement, keeping the best correction.
These settings are kept for all experiments in our paper.

Algorithm~\ref{alg:gcd2} requires a maximum number of iterations ($J$). For SICA-WGF, we set it to 10, and for SICA-RF, we set it to 30 for AR(7), the speech data and the heart example, and 20 for MNIST data.

\subsection{Implementations of Baseline Methods}
\label{sec.baselines}
\begin{itemize}
    \item FastICA: We use the implementation provided by  \verb|sklearn.decomposition.FastICA|. 
        \begin{itemize}
            \item We use the default hyperparameters provided with the implementation and set the maximum iteration to 20000. 
        \end{itemize} 
    \item Least-squares ICA: We use the implementation provided by the original authors \url{https://ibis.t.u-tokyo.ac.jp/suzuki/software/LICA/}. 
        \begin{itemize}
            \item We provide a list of kernel bandwidth $\sigma = [0.1, 0.5]$ and regularization parameters $\lambda = [0.01, 0.1]$ for the method to select using the built-in cross-validation procedure. We set the number of basis functions for the RBF kernel to 50 to speed up the computation. 
        \end{itemize}
    \item iVAE: We use the implementation provided by the original authors \url{https://github.com/ilkhem/icebeem/tree/master/models/ivae}. 
    \begin{itemize}
        \item We use the default hyperparameters provided with the implementation, except the number of epochs was changed from 7e4 to 2e4, which does not seem to impact the performance but saves a significant amount of time. 
    \end{itemize}
    \item Permutation Contrastive Learning (PCL): We implemented it ourselves, and were inspired by \url{https://github.com/RyanBalshaw/Nonlinear_ICA_implementations}. 
    \begin{itemize}
        \item We tried various settings of neural network structures and settled on an MLP with two hidden layers and 256 neurons in each layer for the most reliable performance.
    \end{itemize}
    \item AMUSE \citep{tong_indeterminacy_1991}: We implemented it ourselves: the observations are whitened, and the whitened signals are rotated by the eigenvectors of their symmetrized covariance at lag 1. Images are flattened in raster order.

\end{itemize}
LICA, PCL and iVAE receive the observations scaled to $[0, 1]$ per signal. The code for these methods can be found under the \verb|competitors| folder in the attached code repository.

\subsection{Heart Example}
\label{app.heart}

The true signals in Figure~\ref{fig:illus.example} trace a heart curve,
\begin{align*}
    \rs_{1,t} &= 16 \sin^3(u_t) + \epsilon_t, & \epsilon_t &\sim \mathrm{uniform}(0, 1), \\
    \rs_{2,t} &= 13 \cos(u_t) - 5 \cos(2u_t) - 2 \cos(3u_t) - \cos(4u_t) + \epsilon'_t, & \epsilon'_t &\sim \mathrm{uniform}(0, 2),
\end{align*}
where $u_1, \dots, u_T$ are evenly spaced on $[0, 5\pi]$ ($T = 1024$), so that the curve is traversed 2.5 times. The two signals are dependent through the common variable $u_t$, as in \eqref{eq.example}, but each of them is predictable from its own neighbors in time. We mix them with $M = 10$ steps of the linear mixing function of Section~\ref{sec.exp.ar} and scale the mixed signals to $[0, 1]$. Both SICA variants use the same context window as for AR(7). For SICA-WGF, we use the step size of MNIST, $\eta = 0.2$: with $\eta = 0.5$, the MCC peaks within the first iterations and then decreases. SICA-RF uses the settings of Appendix~\ref{sec.hyper}. iVAE uses $u_t$ as its auxiliary variable.

Table~\ref{tab.heart} reports the MCC over 20 independent runs, each with new noise $\epsilon_t, \epsilon'_t$ and a new initialization of every method. Both SICA variants recover the signals, SICA-WGF almost exactly. FastICA, LICA, PCL and iVAE have a lower MCC than the mixed signals themselves. AMUSE, however, recovers the signals exactly: the two signals have different temporal correlations, which a linear method that uses the temporal structure can exploit when the mixing is linear.

\begin{table}[h]
\centering
\caption{MCC on the heart example over 20 independent runs (mean $\pm$ standard error).}
\label{tab.heart}
\begin{tabular}{lc}
\hline
Method & MCC \\
\hline
Mixed signals & $0.875 \pm 0.000$ \\
FastICA & $0.775 \pm 0.001$ \\
LICA & $0.714 \pm 0.001$ \\
PCL & $0.687 \pm 0.027$ \\
iVAE & $0.386 \pm 0.017$ \\
AMUSE & $\mathbf{0.999 \pm 0.000}$ \\
SICA-WGF & $0.993 \pm 0.001$ \\
SICA-RF & $0.953 \pm 0.006$ \\
\hline
\end{tabular}
\end{table}

\subsection{Reproducing Results}
\label{sec.repro}

To reproduce illustrative results, please run: \verb|demo_heart.py|, \verb|demo_AR7.py|, \verb|demo_MNIST.py| (\verb|--flow wgf-tre|), \verb|demo_rep_AR7.py| and \verb|demo_rep_MNIST.py|; in the repeated-run scripts, SICA-WGF and SICA-RF are the methods \verb|wgf-tre| and \verb|rf-lowrank|, and the AR(7) data of Section~\ref{sec.exp.ar} are selected by \verb|--identifiable|. The baselines of Figures~\ref{fig.ar} and \ref{fig.mnist} are run by \verb|demo_rep_competitors.py| (AMUSE with \verb|--methods amuse|, iVAE with \verb|--methods ivae|), and all figures are drawn by \verb|make_paper_figures.py|. The speech experiments (Figure~\ref{fig.speech}) are run by \verb|demo_rep_speech.py| (\verb|--corpus librispeech| or \verb|arctic|) and \verb|demo_rep_competitors.py| (\verb|--dataset speech| or \verb|arctic|). Table~\ref{tab.heart} is produced by \verb|demo_rep_heart.py|.

\section{Illustration of Sequence SICA}
\label{app.seq.illus}

\begin{figure}[h]
    \centering
    \includegraphics[width=0.6\textwidth]{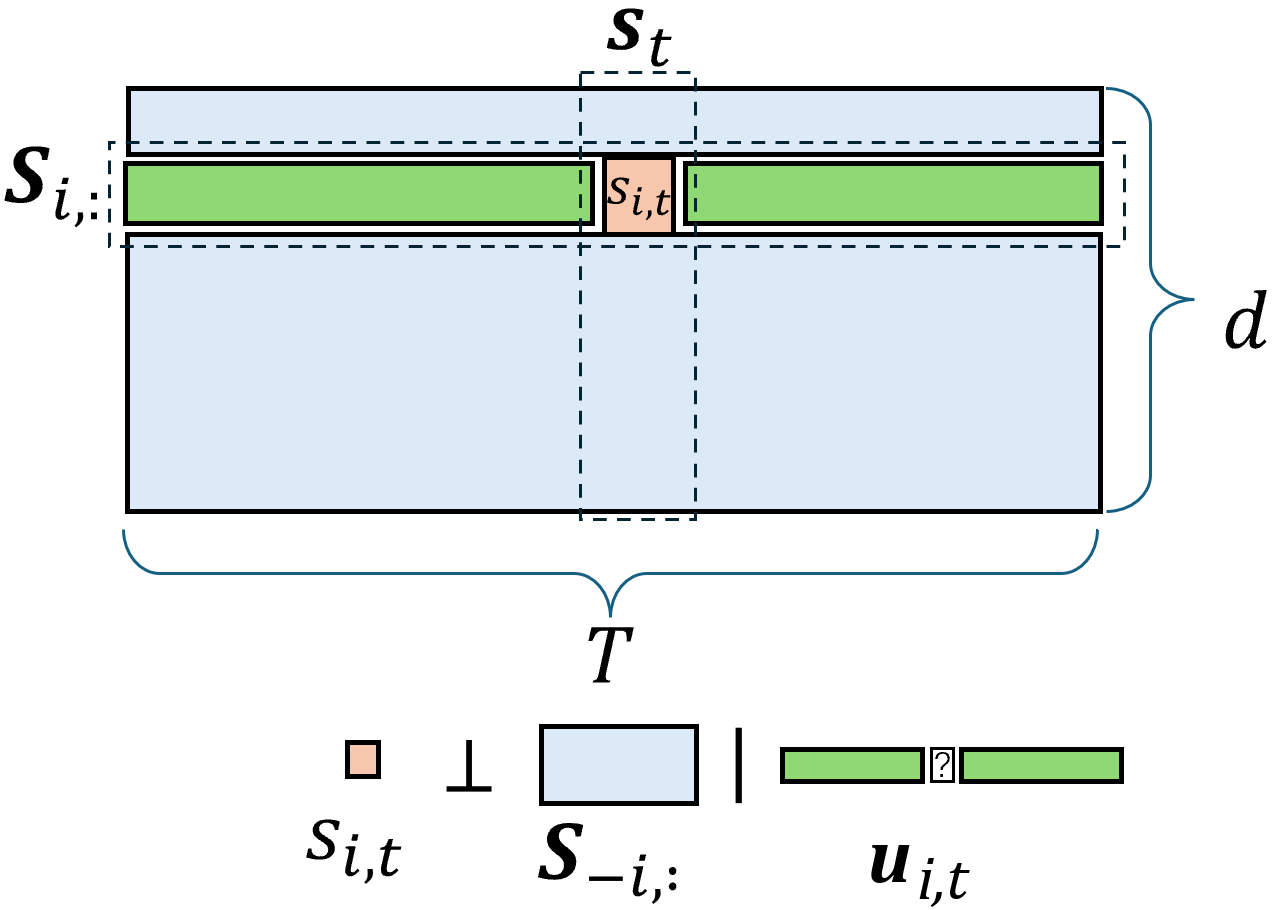}
    \caption{The SICA assumption \eqref{eq:sica:independence} on sequence data $\rmS \in \R^{d \times T}$: given the context $\rvu_{i,t}$, i.e., the $i$-th row without its $t$-th entry (green), the entry $\rs_{i,t}$ (orange) is independent of the other rows $\rmS_{-i,:}$ (blue).}
    \label{fig.seq.sica}
\end{figure}

\section{Mixing and De-mixing Flows}
\label{sec.mix.demix.flows}

In classic linear ICA literature, the observation is an instantaneous cross-dimension mixing $\rvx = \mA \rvs$, where $\mA$ is the \emph{mixing matrix}.
This operation can alternatively be interpreted as the outcome of a continuous-time process: if $\mA$ has a real matrix logarithm $\mV = \log(\mA) \in \R^{d \times d}$ (e.g., when $\mA$ has no negative real eigenvalues; see \citealp{culver_existence_1966}), then the solution of the linear ODE
\begin{align}
\label{eq.demixingflow}
    \frac{\mathrm{d} \rvs(\tau)}{\mathrm{d} \tau} = \mV \rvs(\tau), \qquad \rvs(0) = \rvs,
\end{align}
satisfies $\rvs(1) = \exp(\mV) \rvs = \mA \rvs = \rvx$.
Intuitively, such a mixing matrix can stretch and rotate the signals, but cannot mirror them, since $\det \mA = e^{\operatorname{tr} \mV} > 0$.

Generalizing the linear velocity field to a non-linear function $\mV'$ gives the \emph{mixing flow}
\begin{align}
\label{eq.demixingflow2}
    \frac{\mathrm{d} \rvs(\tau)}{\mathrm{d} \tau} = \mV'(\rvs(\tau)),
\end{align}
and we consider mixing functions of this form, i.e., $\rvx = \rvs(1)$.
Such mixing can be undone by a flow, which motivates learning ODEs as de-mixing functions.
Since the non-linear ODE generally lacks a closed-form solution, recovering the sources from $\rvx = \rvs(1)$ requires learning a velocity field and simulating the process in reverse; we refer to such a backward process as a \emph{de-mixing flow}.

\section{The Complete Algorithm}
\label{app.alg}

Algorithm~\ref{alg:gcd2} gives SICA-WGF and SICA-RF in full, including the projection step of Algorithm~\ref{alg:gcd}.

\begin{algorithm}[h]
\caption{SICA-$\{\text{WGF}, \text{RF}\}$}
\label{alg:gcd2}
\begin{algorithmic}[1]
\Require A mixed multi-dimensional sequence $\rmX$; choice of flow $\in \{\mathrm{WGF}, \mathrm{RF}\}$
\Ensure Recovered signals $\rmZ^{(J)}$
\State $\rmZ^{(0)} \gets \rmX$
\For{$j = 1, \dots, J$}
    \State $\mathcal{D} \gets \{(\rvz_t^{(j-1)}, \rmC_t^{(j-1)})\}_{t=1}^{T}$ \Comment{samples from $\rho_0$}
    \For{$i = 1, \dots, d$}
        \State $\mathcal{D}_i' \gets \{(\rz_{i, \sigma_i(t)}^{(j-1)}, \rvc_{i, \sigma_i(t)}^{(j-1)})\}_{t=1}^{T}$, $\sigma_i$ a random permutation of $[T]$ \Comment{samples from $p(\rz_i, \rvc_i)$}
    \EndFor
    \If{choice of flow is WGF}
        \State Estimate $\vv^*$ in \eqref{eq.log.joint.ratio} using $\mathcal{D}$ and $\mathcal{D}'_1, \dots, \mathcal{D}'_d$, and fit $\vu$ in \eqref{eq.rf.proj} to $\{\vv^*(\rvz_t^{(j-1)}; \rmC_t^{(j-1)})\}_{t \in [T]}$
        \State $\rvz_t^{(j)} \gets \rvz_t^{(j-1)} + \eta\, \vu(\rvz_t^{(j-1)})$ for all $t \in [T]$ \Comment{one Euler step of the context-free field}
    \Else
        \State Estimate $\vv^*$ in \eqref{eq.rec.obj} using $\mathcal{D}$ and $\mathcal{D}'_1, \dots, \mathcal{D}'_d$
        \State Solve the ODE of $\vv^*(\cdot, \rmC_t^{(j-1)}, \rmC_t^{(j-1)}, \tau)$ from $\rvz_t^{(j-1)}$ for all $t \in [T]$, and fit $\vu_\tau$ in \eqref{eq.rf.proj} on the trajectories
        \State $\rvz_t^{(j)} \gets \psi_1(\rvz_t^{(j-1)})$ for all $t \in [T]$ \Comment{flow of the context-free field $\vu_\tau$}
    \EndIf
\EndFor
\State \Return $\rmZ^{(J)}$
\end{algorithmic}
\end{algorithm}

\clearpage
\section{Additional Experimental Results}
\label{app.exp.figs}

Figures~\ref{fig.ar}, \ref{fig.mnist} and~\ref{fig.speech} show the illustrative examples and the MCC under both linear and non-linear mixing for the experiments of Section~\ref{sec.exp}.

\begin{figure}[h]
\centering
\subfigure[Mixed Signal (blue), Recovered Signal (red) and True Signal (black)]{
\includegraphics[width=.315\textwidth]{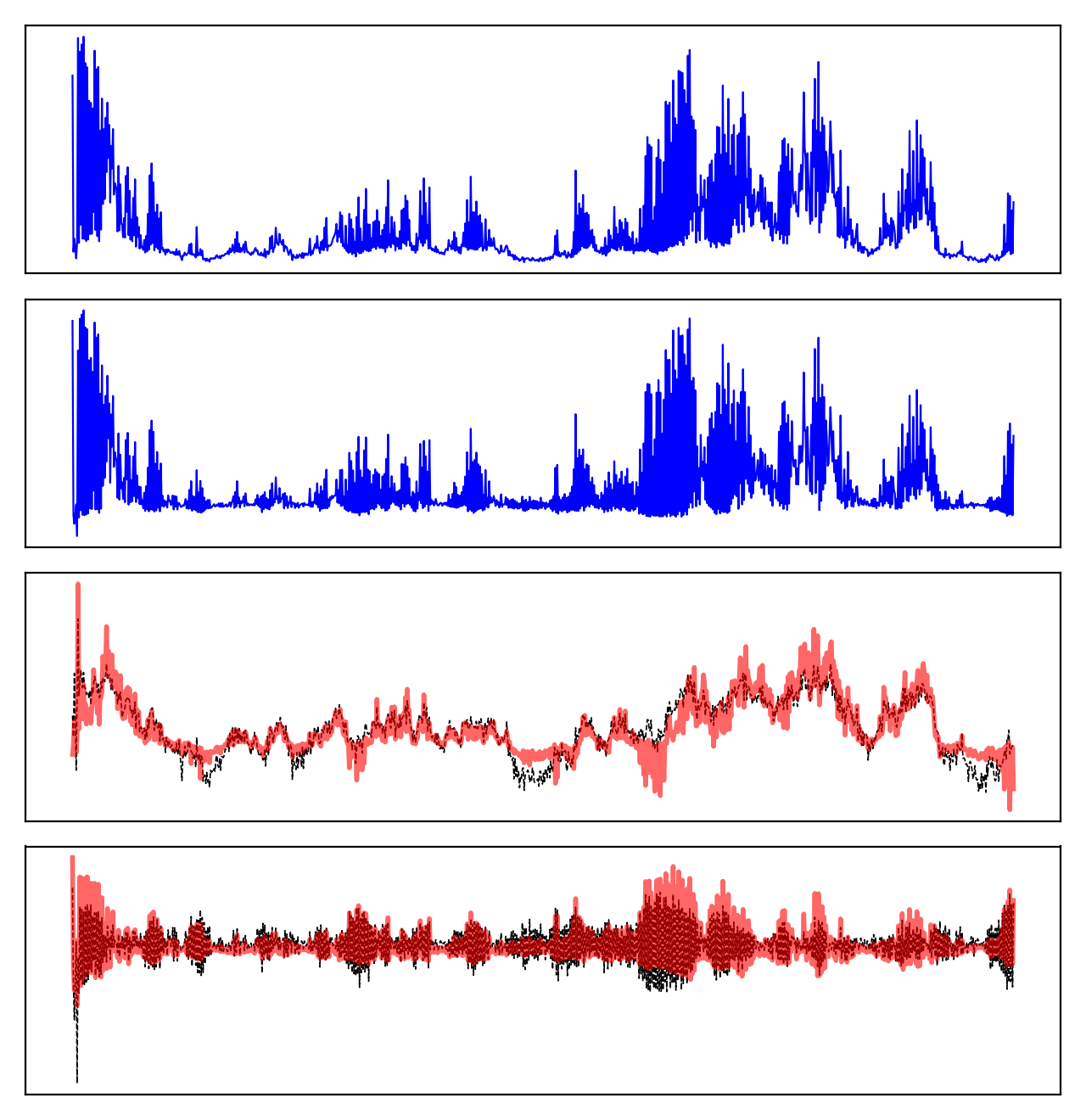}
\label{fig.mcc.demo}
}
\subfigure[Linear Mixing $\vf$ on AR(7)]{
\includegraphics[width=.315\textwidth]{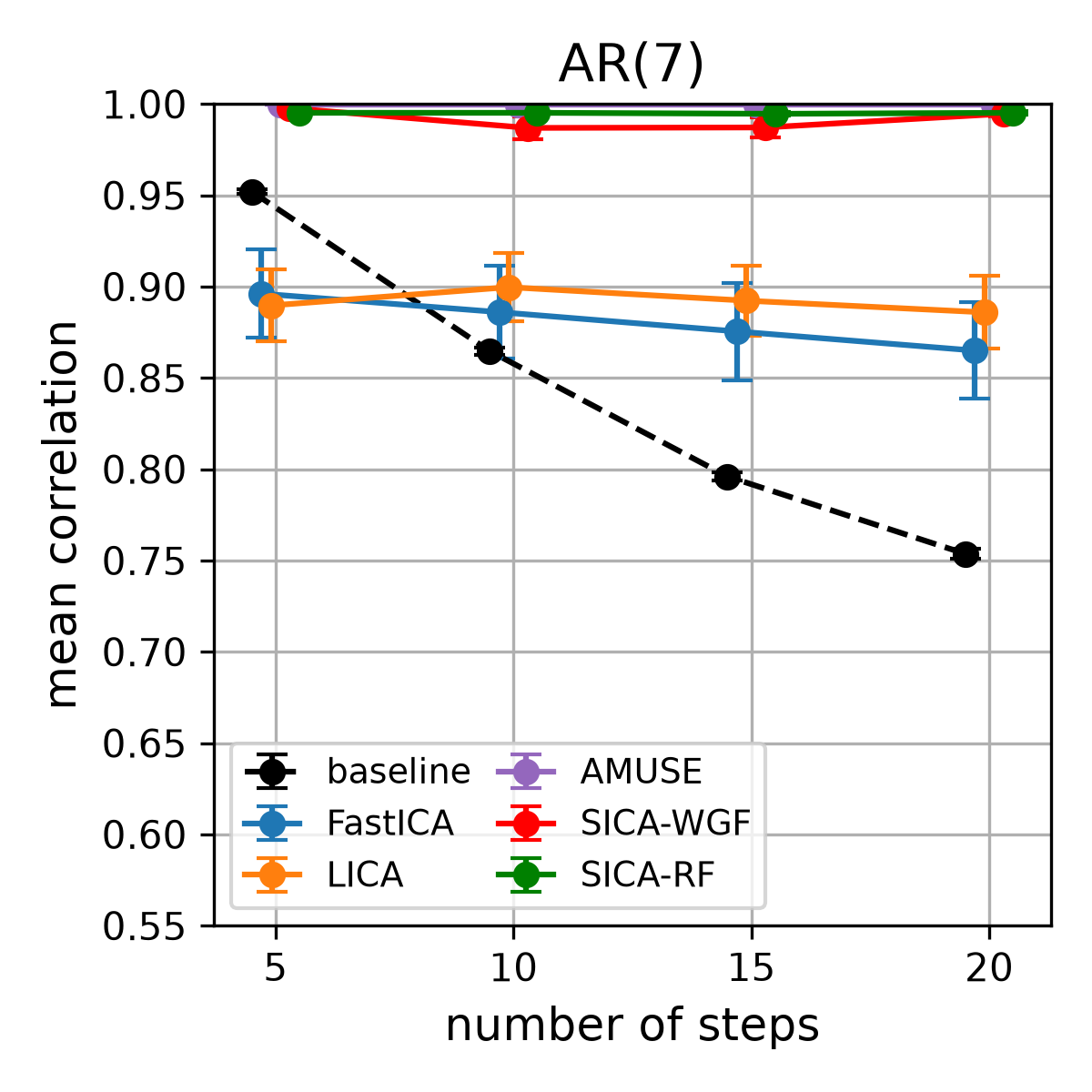}
\label{fig.mcc.linear}
}
\subfigure[Non-linear Mixing $\vf$ on AR(7)]{
\includegraphics[width=.315\textwidth]{figures/ICAFlow16_nonlinear.png}
\label{fig.mcc.nonlinear}
}

\caption{AR(7) dataset. Left: Illustrative Example. Middle and Right: MCC over various mixing steps. The higher the better. MCC is measured over 20 independent runs, and error bars represent the standard error. }
\label{fig.ar}
\end{figure}

\begin{figure}[h]
\centering
\subfigure[Mixed Signal (1st row), True Signal (2nd row), Signal Recovered by SICA-WGF (3rd row) and by FastICA (4th row)]{
\centering
\includegraphics[width=.3\textwidth]{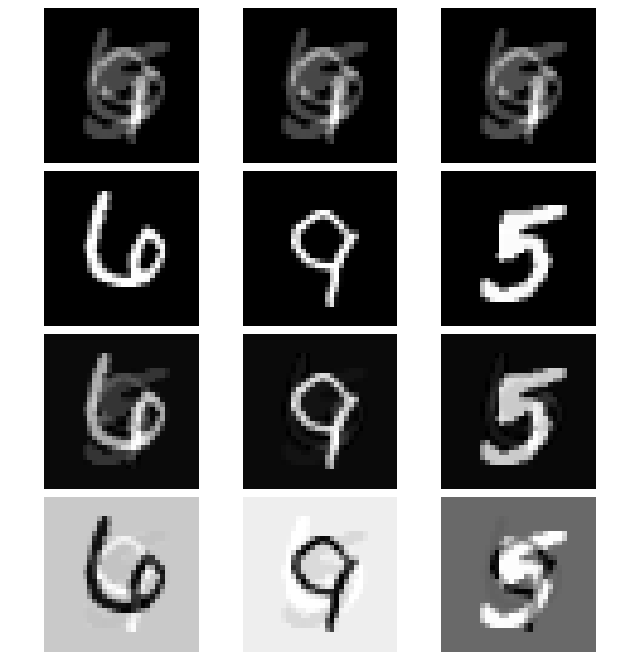}
\label{fig.mnist.demo}
}
\subfigure[Linear Mixing $\vf$ on MNIST]{
\includegraphics[width=.3\textwidth]{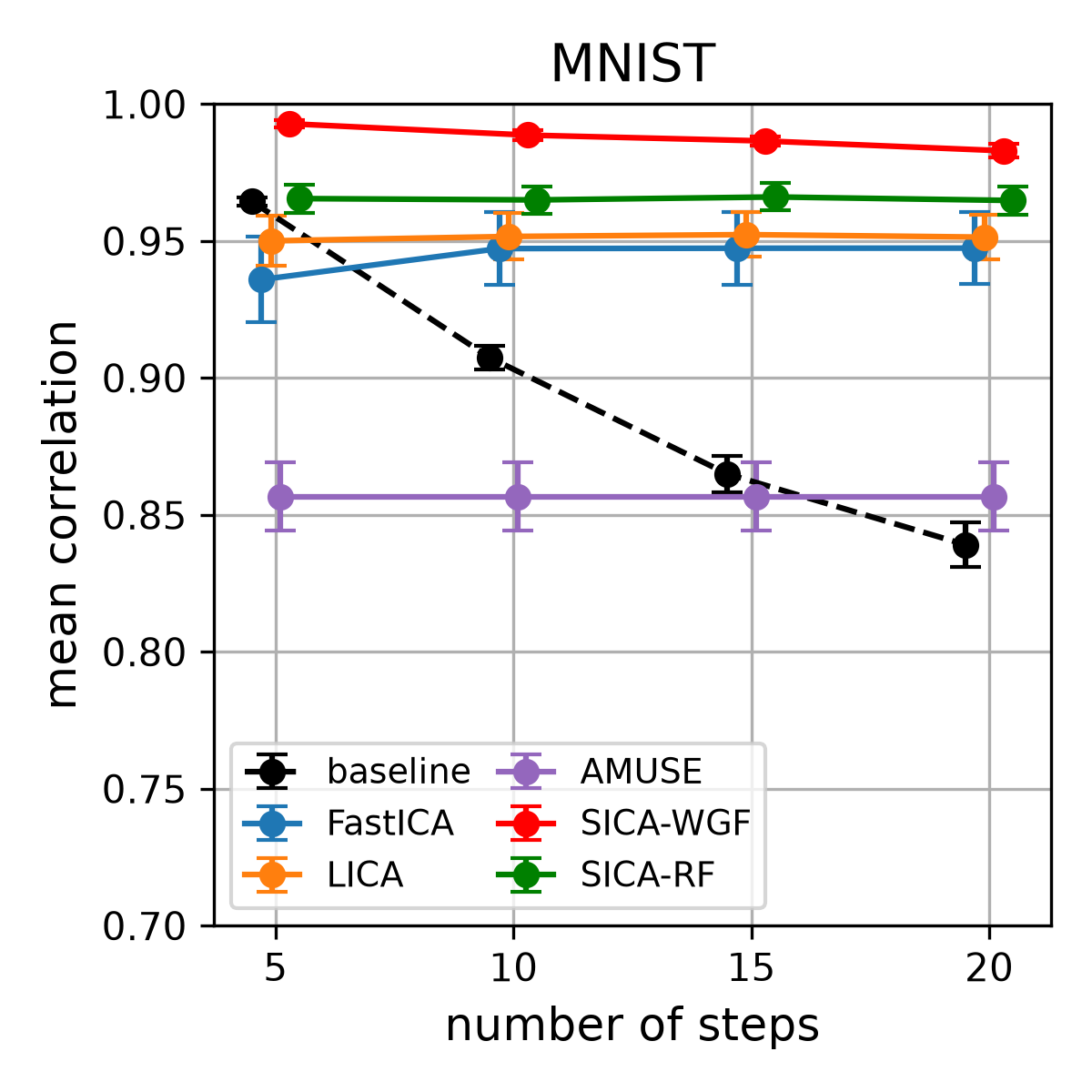}
\label{fig.mnist.linear}
}
\subfigure[Non-linear Mixing $\vf$ on MNIST]{
\includegraphics[width=.3\textwidth]{figures/ICAFlow17_nonlinear.png}
\label{fig.mnist.nonlinear}
}
\caption{MNIST dataset. Left: Illustrative Example. Middle and Right: MCC over various mixing steps. MCC is measured over 20 independent runs, and error bars represent the standard error. }
\label{fig.mnist}
\end{figure}

\paragraph{MNIST: illustrative example.}
First, we select three images ($d=3$) showing the digits 6, 9 and 5 (Figure~\ref{fig.mnist.demo}), and apply 20 steps of non-linear mixing. It can be seen that all observed channels are visually almost identical after mixing. We then use SICA-WGF to disentangle these images. It can be seen that SICA recovers the three digits 6, 9, and 5 (MCC 0.97), with faint residues of the other digits. The separations given by FastICA are not as clear (MCC 0.95): its third channel still mixes the 5 with the other digits. On this example, SICA-RF (MCC 0.94) is slightly below FastICA.

One interesting observation is that the third recovered channel by FastICA contains both white and dark strokes on a gray background, whereas the digits recovered by SICA keep the polarity of the sources (white on black).
Linear ICA has permutation and scaling ambiguities, including sign flips; a global polarity flip alone is therefore not a separation error. Theorem~\ref{thm.ident} allows the broader ambiguity of component-wise invertible maps and does not guarantee polarity preservation either. The preserved polarity in this example is an empirical observation; whether it reflects an implicit bias of incremental flow refinements remains to be studied.

\begin{figure}[h]
\centering
\subfigure[Linear]{
\includegraphics[width=.22\textwidth]{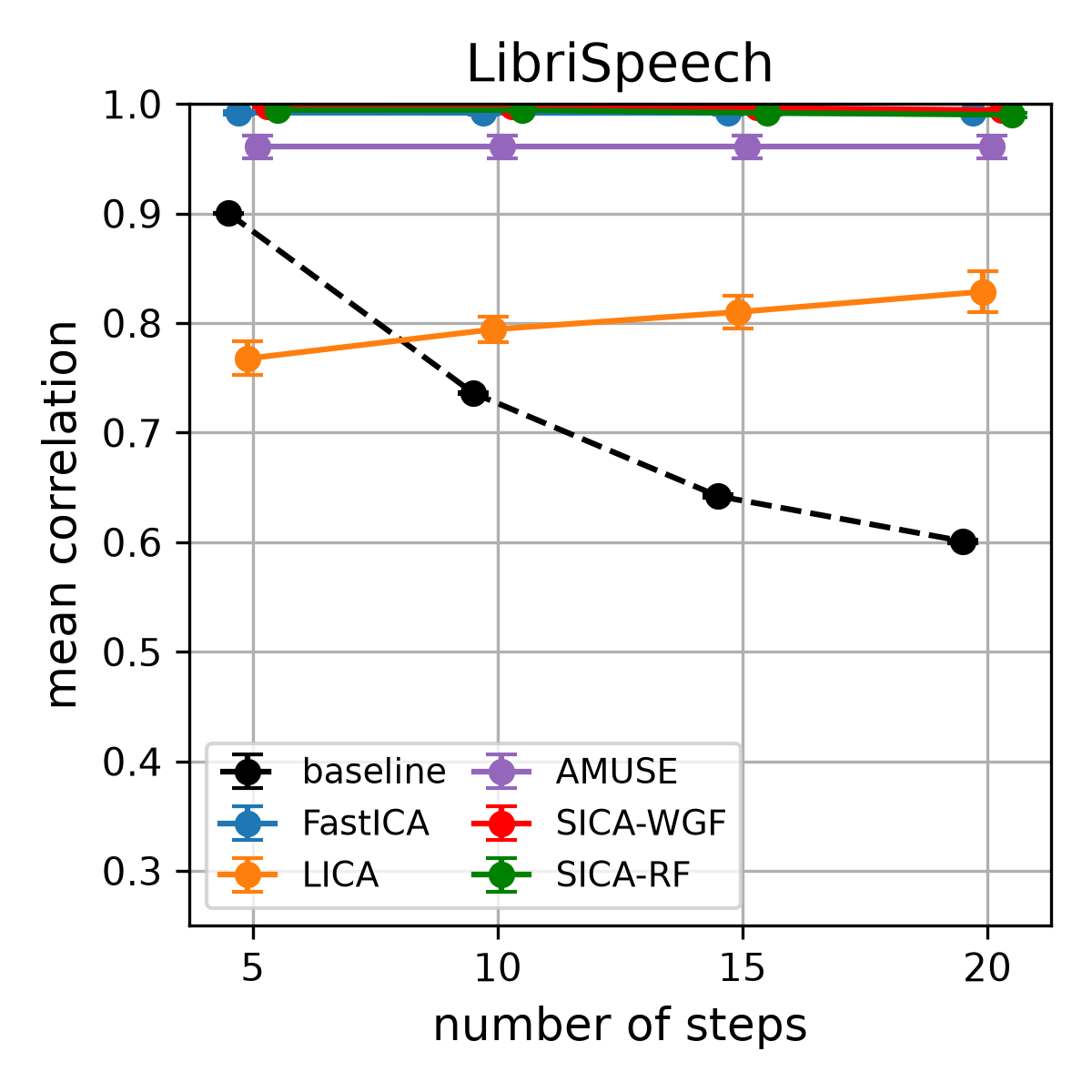}
\label{fig.speech.libri.linear}
}
\subfigure[Non-linear]{
\includegraphics[width=.22\textwidth]{figures/librispeech_nonlinear.png}
\label{fig.speech.libri.nonlinear}
}
\subfigure[Linear]{
\includegraphics[width=.22\textwidth]{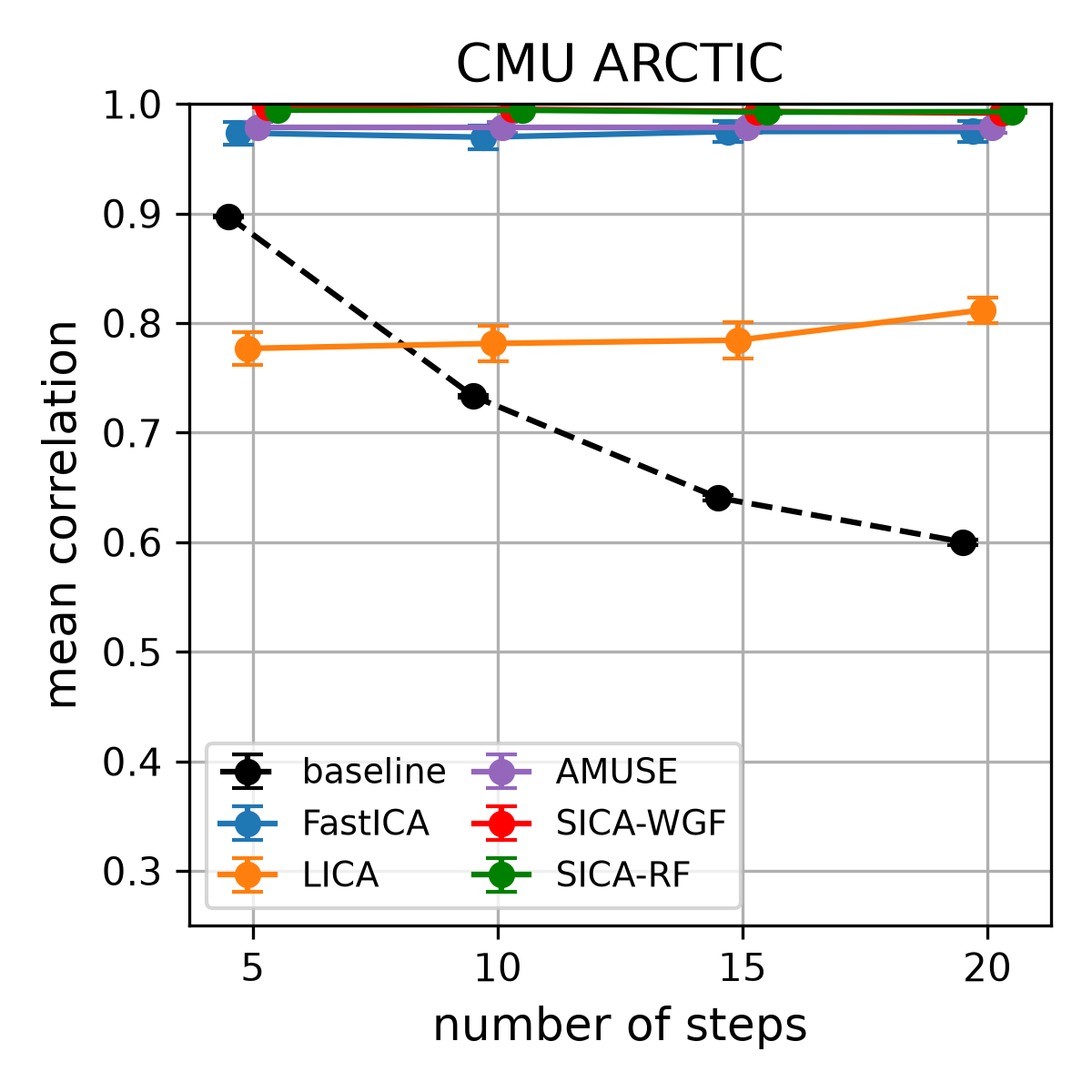}
\label{fig.speech.arctic.linear}
}
\subfigure[Non-linear]{
\includegraphics[width=.22\textwidth]{figures/arctic_nonlinear.png}
\label{fig.speech.arctic.nonlinear}
}
\caption{Speech mixtures ($d=3$) of LibriSpeech (a, b) and CMU ARCTIC (c, d) under linear and non-linear mixing $\vf$. MCC over various mixing steps, measured over 20 independent runs; error bars represent the standard error.}
\label{fig.speech}
\end{figure}

\end{document}

%% file: math_commands.tex
\usepackage{amsmath,amsfonts,bm}

 \newcommand{\ind}{\perp\!\!\!\!\perp}

\def\eqref#1{equation~\ref{#1}}
\def\Eqref#1{Equation~\ref{#1}}

\def\1{\bm{1}}

\def\rs{{\textnormal{s}}}

\def\ru{{\textnormal{u}}}

\def\rx{{\textnormal{x}}}

\def\rz{{\textnormal{z}}}

\def\rvc{{\mathbf{c}}}

\def\rve{{\mathbf{e}}}

\def\rvu{{\mathbf{i}}}

\def\rvs{{\mathbf{s}}}

\def\rvu{{\mathbf{u}}}

\def\rvw{{\mathbf{w}}}
\def\rvx{{\mathbf{x}}}
\def\rvy{{\mathbf{y}}}
\def\rvz{{\mathbf{z}}}

\def\rmC{{\mathbf{C}}}

\def\rmR{{\mathbf{R}}}
\def\rmS{{\mathbf{S}}}

\def\rmU{{\mathbf{U}}}

\def\rmX{{\mathbf{X}}}
\def\rmY{{\mathbf{Y}}}
\def\rmZ{{\mathbf{Z}}}

\def\va{{\bm{a}}}
\def\vb{{\bm{b}}}
\def\vc{{\bm{c}}}

\def\ve{{\bm{e}}}
\def\vf{{\bm{f}}}
\def\vg{{\bm{g}}}
\def\vh{{\bm{h}}}

\def\vk{{\bm{k}}}
\def\vl{{\bm{l}}}
\def\vm{{\bm{m}}}

\def\vq{{\bm{q}}}

\def\vs{{\bm{s}}}

\def\vu{{\bm{u}}}
\def\vv{{\bm{v}}}
\def\vw{{\bm{w}}}
\def\vx{{\bm{x}}}
\def\vy{{\bm{y}}}
\def\vz{{\bm{z}}}

\def\mA{{\bm{A}}}
\def\mB{{\bm{B}}}
\def\mC{{\bm{C}}}
\def\mD{{\bm{D}}}

\def\mG{{\bm{G}}}

\def\mI{{\bm{I}}}
\def\mJ{{\bm{J}}}

\def\mM{{\bm{M}}}

\def\mO{{\bm{O}}}
\def\mP{{\bm{P}}}
\def\mQ{{\bm{Q}}}

\def\mU{{\bm{U}}}
\def\mV{{\bm{V}}}
\def\mW{{\bm{W}}}
\def\mX{{\bm{X}}}

\def\mZ{{\bm{Z}}}

\DeclareMathAlphabet{\mathsfit}{\encodingdefault}{\sfdefault}{m}{sl}
\SetMathAlphabet{\mathsfit}{bold}{\encodingdefault}{\sfdefault}{bx}{n}

\def\sR{{\mathbb{R}}}

\newcommand{\E}{\mathbb{E}}

\newcommand{\R}{\mathbb{R}}

\newcommand{\KL}{D_{\mathrm{KL}}}

\DeclareMathOperator*{\argmin}{arg\,min}